\documentclass{article}

\usepackage{microtype}
\usepackage{graphicx}
\usepackage{subfigure}
\usepackage{booktabs} 
\usepackage{caption}
\usepackage{subcaption}
\usepackage{wrapfig}

\newcommand\sA{\ensuremath{\mathcal{A}}}

\newcommand\sR{\ensuremath{\mathcal{R}}}
\newcommand\sS{\ensuremath{\mathcal{S}}}

\newcommand{\embed}[1]{\ensuremath{\phi(#1)}}
\newcommand{\valuematrix}[0]{\ensuremath{W_\text{V}}}
\newcommand{\kq}[0]{\ensuremath{W_\text{KQ}}}
\newcommand{\ntp}[4]{\ensuremath{f(#3,#4;#1,#2)}}
\newcommand{\dpre}[0]{\ensuremath{D_\text{pre}}}
\newcommand{\transpose}[1]{\ensuremath{#1^\top}}
\newcommand{\salience}[3]{\ensuremath{S(#1,#2,#3)}}
\newcommand{\srel}[0]{\ensuremath{s_\text{rel}}}
\newcommand{\ansset}[1]{\ensuremath{\mathcal{A}^{#1}}}
\newcommand{\prel}[0]{\ensuremath{p_\text{rel}}}
\newcommand{\subjsect}[0]{\ensuremath{\mathcal{S}}}
\newcommand{\tokenset}[0]{\ensuremath{\mathcal{T}}}
\newcommand{\fextended}[1]{\ensuremath{f(#1, W^Q, W^K,W^V)}}
\newcommand{\Att}[1]{\ensuremath{\text{Att}_{#1}}}
\newcommand{\keymatrix}[0]{\ensuremath{W^K}}
\newcommand{\softmax}[1]{\ensuremath{\sigma(#1)}}
\newcommand{\querymatrix}[0]{\ensuremath{W^Q}}
\newcommand{\selfatt}[3]{\ensuremath{\text{Self-Att}(#3; #1, #2)}}
\newcommand{\selfattext}[4]{\ensuremath{\text{Self-Att}(#4; #1,#2,#3)}}
\newcommand{\partialderiv}[2]{\ensuremath{\frac{\partial#1}{\partial#2}}}

\newcommand{\att}[1]{\ensuremath{\text{Att}_{#1}}}

\usepackage{hyperref}

\usepackage[accepted]{icml2024}

\usepackage{amsmath}
\usepackage{physics}
\usepackage{amssymb}
\usepackage{mathtools}
\usepackage{amsthm}

\DeclareMathOperator*{\argmax}{arg\,max}

\usepackage[capitalize,noabbrev]{cleveref}

\theoremstyle{plain}
\newtheorem{theorem}{Theorem}[section]

\theoremstyle{definition}
\newtheorem{definition}[theorem]{Definition}
\newtheorem{assumption}[theorem]{Assumption}
\theoremstyle{remark}

\usepackage{bbm}
\usepackage[textsize=tiny]{todonotes}

\icmltitlerunning{Understanding Finetuning for Factual Knowledge Extraction}

\begin{document}

\twocolumn[
\icmltitle{Understanding Finetuning for Factual Knowledge Extraction}



\icmlsetsymbol{equal}{*}

\begin{icmlauthorlist}
\icmlauthor{Gaurav Ghosal}{cmu}
\icmlauthor{Tatsunori Hashimoto}{stanford}
\icmlauthor{Aditi Raghunathan}{cmu}

\end{icmlauthorlist}

\icmlaffiliation{cmu}{Department of Machine Learning, Carnegie Mellon University, Pittsburgh, USA}
\icmlaffiliation{stanford}{Department of Computer Science, Stanford University, Stanford, USA}
\icmlcorrespondingauthor{Gaurav Ghosal}{gghosal@andrew.cmu.edu}

\icmlkeywords{Large Language Models, Factuality}

\vskip 0.3in
]
\printAffiliationsAndNotice{}


\begin{abstract}
In this work, we study the impact of QA fine-tuning data on downstream factuality. We show that fine-tuning on lesser-known facts that are poorly stored during pretraining yields significantly worse factuality than fine-tuning on well-known facts, even when all facts are seen during pretraining. 
We prove this phenomenon theoretically, showing that training on lesser-known facts can lead the model to ignore subject entity names and instead output a generic plausible response \emph{even when the relevant factual knowledge is encoded in the model}. On three question answering benchmarks (PopQA, Entity Questions, and MMLU) and two language models (Llama-2-7B and Mistral-7B), we find that (i) finetuning on a completely factual but lesser-known subset of the data deteriorates downstream factuality (5-10\%) and (ii) finetuning on a subset of better-known examples matches or outperforms finetuning on the entire dataset. Ultimately, our results shed light on the interaction between pretrained knowledge and finetuning data and demonstrate the importance of taking into account \emph{how} facts are stored in the pretrained model when fine-tuning for knowledge-intensive tasks.
\end{abstract}

\section{Introduction}
Large language models store large amounts of factual knowledge in their weights during pretraining \cite{jiang2020know, petroni2019language, mallen2023trust}. As a result, they have shown promise on a variety of knowledge intensive tasks, including factual question-answering \cite{roberts-etal-2020-much, radford2019language}. However, these abilities are unreliable and language models are prone to generate plausible, but incorrect responses to queries \cite{huang2023survey}. 

A natural avenue to improve factuality is via fine-tuning, as studied in several recent works~\cite{kazemi2023understanding,joshi2023personas, ouyang2022training,tian2023finetuning, yang2023alignment}. Multiple works, however, have shown that language models answer questions incorrectly even when they know the right answer, suggesting that current approaches to fine-tuning may be suboptimal \cite{burns2022discovering,li2023inferencetime,liu2023cognitive}. In order to achieve better fine-tuning or uncover the ceiling of such approaches, we need to understand what factors determine the performance of fine-tuning. What is the mechanism by which fine-tuning improves factuality?

We can distill prior understanding of this question into three factors. \citet{joshi2023personas} posits that fine-tuning on truthful data influences the model to adopt a credible \emph{persona}. This theory suggests that ensuring the \emph{factual accuracy} of the finetuning data is sufficient for downstream factuality. Another view from~\citet{kazemi2023understanding} and \citet{allenzhu2023physics} is that fine-tuning familiarizes the pretrained model with the QA format, which varies from the way that facts are observed during pretraining. This implies that finetuning examples should cover question formats likely to be seen during testing. Finally, \citet{schulmantalk} and \citet{yang2023alignment} hypothesize that fine-tuning examples must be drawn from facts that the model sees during pretraining.

In this work, we find that the impact of fine-tuning examples depends on \emph{how well} they are stored in the model, beyond simply their factuality or whether they are grounded in the pretraining corpus. Concretely, fine-tuning on QA examples about facts that the pretrained model knows well significantly improves factuality. Conversely, fine-tuning on QA examples regarding less well-encoded facts \emph{actively harms} downstream factuality, causing the model to incorrectly respond to questions it could otherwise get right. We make this finding in a synthetic setting, after ensuring that all QA examples are factually accurate, representative of the downstream task, and seen during pretraining.

Why does the encoding of facts seen in finetuning affect factuality downstream? We propose the following intuitive mechanism. When presented with a factual question, a model can either respond using relevant memorized knowledge or leverage more general ``shortcuts" that enable it to propose a plausible, but incorrect response. For example, when asked about a person's occupation, a language model could potentially take the shortcut of responding with a word that is generally associated with occupations (i.e. actor). If 
shortcut usage is reinforced during fine-tuning, this can drown out the influence of memorized knowledge, causing the model to behave less factually on test data. Our observations suggest that the composition of the fine-tuning data controls which mechanism is amplified: less well-known facts can lead to more aggressive use of shortcuts. We conceptually illustrate our hypothesis in Figure \ref{fig:conceptual}.

In Section \ref{sec:thy}, we prove this intuition in a one-layer transformer. We introduce a quantity termed \emph{factual salience} that measures how well a fact is learned by the one-layer transformer. Next, we demonstrate that a one-layer transformer can resort to using shortcuts through \emph{attention imbalance}: attending only to more general tokens (for example those that specify the question type) rather than the specific entities in the question. We prove that \emph{fine-tuning gradients on less salient facts contribute to the formation of attention imbalance}, while those on more salient facts counteract it. Furthermore, we show the effect of attention imbalance is amplified when looking at downstream performance on less well-known facts. Our results have a counterintuitive consequence: for less well-known facts, it is \emph{worse} to fine-tune on similar less well-known facts and better to fine-tune on a different distribution of more well-known facts.

We test the implications of our analysis on three real-world QA datasets (PopQA, MMLU, and Entity Questions) and two LLM models (Llama-2-7B and Mistral-7B). As predicted by our theory, we find that fine-tuning on well-known knowledge (top 50\%) outperforms fine-tuning on less well-known knowledge (bottom 50\%) by $7\%$ on MMLU, $6\%$ on PopQA, and $4\%$ on EntityQuestions. Moreover, we can match the performance of fine-tuning on the entire dataset by finetuning on just the top $50\%$. On MMLU, we find that finetuning on the top $30\%$ well-known facts \emph{outperforms} finetuning on the entire dataset by up to 2.5\%. 

To summarize, via theory and experiments, we uncover an important factor that determines the effect of finetuning on downstream factuality---how well the finetuned facts are encoded in the pretrained model. Beyond a conceptual understanding, our findings have immediate practical considerations for finetuning data curation: it can suffice to focus on a smaller number of well-known facts even when trying to improve factuality on less well-known facts. 

\begin{figure}
    \centering
    \includegraphics[width = 0.5\textwidth]{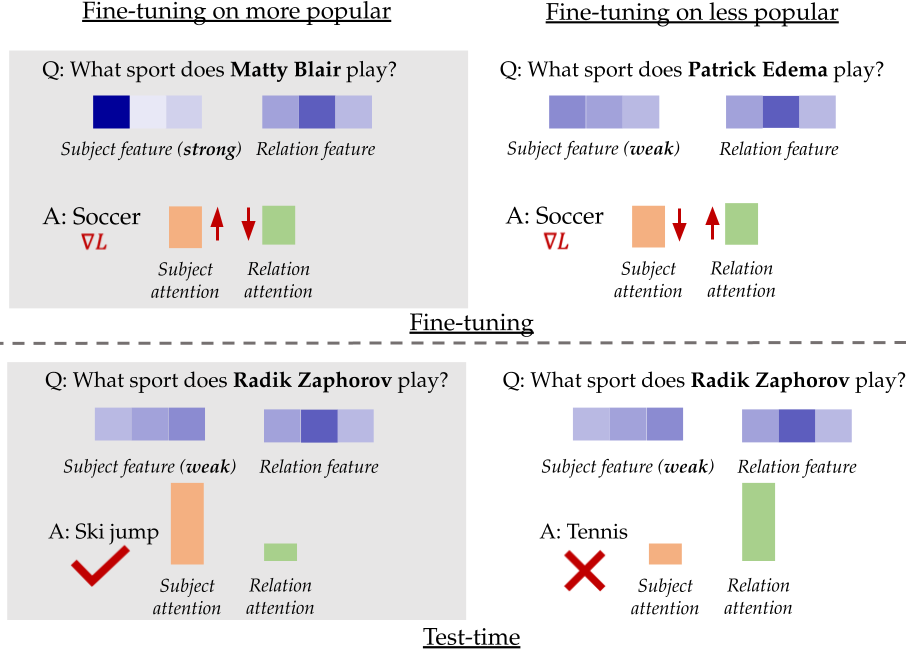}
    \caption{\textbf{Conceptual Mechanism of Finetuning on Popular versus Unpopular Knowledge}. When finetuning on less popular knowledge, the model can learn to heavily upweight relation features which enables it to make a plausible guess about the correct answer. However, training on popular, well-encoded facts discourages this imbalance. At testing time, heavy reliance on relation features can result in less popular knowledge being overwritten.}
    \label{fig:conceptual}
\end{figure}

\section{Preliminaries and Setup}
\label{sec:prelims}
Language models are presented with large quantities of factual knowledge during pretraining, for example in books and articles in the pretraining corpora \cite{jiang2020know, petroni2019language, mallen2023trust}. When users interact with language models, however, it is most desirable for them to extract knowledge in a QA format, which varies from how facts are seen in pretraining. As a result, LLMs must undergo finetuning to learn to apply their pretrained knowledge to these downstream QA tasks. Here, we introduce a formalization of this process which guides our synthetic experiments (Section \ref{sec:synth}) and theory (Section \ref{sec:thy}).

\textbf{Definition of Factual Knowledge} Following prior works  \cite{petroni2019language, elsahar2018t}, we abstractly represent a fact as the mapping from a subject-entity $s$ and relation-type $r$ to an answer $a$. We can represent these mappings as \emph{knowledge triples} $(s,r,a)$ where $s \in \sS$, $r \in \sR$, and $a \in \sA$ and $\sS$, $\sR$, $\sA$ are the set of all subject entities, relations, and answers, respectively. Importantly a single $(s,r,a)$ triple can be expressed in multiple ways in natural language. Here, we model a natural language as the set of sequences of tokens lying in a token set $\mathcal{T}$.

\textbf{Knowledge Formatting Functions} Previously, we presented a definition of factual knowledge but observed that a fact can be presented textually in many formats. We formalize this intuition by introducing the notion of a \emph{formatting function} $g:\sS \times \sR \times \sA \rightarrow \mathcal{T}^{k}$ which maps an $(s,r,a)$ triple to a series of tokens lying in the set $\mathcal{T}$. One such $g$, for example, could map the knowledge triple (USA, capital, Washington D.C.) to the tokenization of the sentence ``The capital of the USA is Washington D.C."

\textbf{Pretraining and Fine-tuning} Now, we are ready to formalize the interaction of the pretraining and finetuning stages. Given a set of knowledge triples $D_{\text{k}}=\{(s,r,a)_{i=1}^{N}\}$ and a \emph{pretraining formatting function}, we generate a pretraining corpus $\dpre = \{g_{\text{pre}}(s,r,a) | (s,r,a) \in D_{\text{k}}\}$. Next, for a \emph{downstream formatting function} $g_\text{down}$, we generate a downstream dataset $D_\text{down} = \{g_{\text{down}}(s,r,a)|(s,r,a) \in D_\text{k}\}$. In practice, the finetuning dataset is often limited relative to pretraining so we partition $D_\text{down}$ into $D_\text{ft}$ and $D_\text{eval}$ and use  $D_\text{ft}$ for finetuning and $D_\text{eval}$ as a held-out test set.

In QA settings, $g_\text{pre}$ presents facts as they would be seen in books and articles, while $g_\text{down}$ presents facts as question-answer pairs (i.e. "What is the capital of the USA? Washington D.C."). The goal of QA finetuning is thus to enable facts observed in the pretraining format to be extracted by prompting in question-answering (QA) format.

\section{Synthetic Experiments}
\label{sec:synth}

In this section, we study the role of fine-tuning data on factuality in a synthetic setup. This setup allows us to investigate the role of the pretraining process, which would be impractical to do in real large language models. 

\subsection{Synthetic Setup} 
We consider the following simulated setup based on the formalism introduced in Section \ref{sec:prelims}. We consider that there is a single token for each subject, relation, and answer. We take $g_\text{pre}(s,r,a) = (s,r,a)$ (i.e. the pretraining formatting function simply maps to the sequence of subject, relation, and answer tokens). To simulate the change in formatting that occurs in downstream tasks, we introduce a \emph{QA-prompt} token  $p_{r}$ for each relation type. The QA-prompt tokens are unseen during pretraining but used in the downstream formatting function: $g_{\text{down}}(s,r,a) = (s, p_{r}, a)$. Thus, during finetuning, the language model must learn to respond to a prompt $(s,p_{r})$ as if it had been prompted with $(s,r)$. Our token space is thus $\mathcal{T} = \sS \cup \sR \cup \sA \cup \{p_{r} | r \in \sR\}$.

During pretraining,  $(s,r,a)$ triples are sampled i.i.d. from the distribution $s \sim \text{Zipf}(\sS), r \sim \text{Unif}(\sR)$ at each step. This modeling choice simulates the fact that pretraining corpora often contain both very popular entities as well as many obscure, rarely seen ones. During fine-tuning, however, we perform standard batch based training on $D_\text{ft}$. We assume that all knowledge sequences presented to the model (in both pretraining and downstream formats) are consistent with the ground truth $(s,r,a)$ triples in $D_\text{k}$. This allows us to study the role of finetuning data \emph{beyond factual correctness} as is the focus of prior work \cite{joshi2023personas}.

Finally, we emphasize that all facts in the downstream fine-tuning ($D_\text{ft}$) and test datasets ($D_\text{eval}$) are present in \dpre. As a result, our simulation results do not arise from the impact of finetuning on new knowledge as has been hypothesized in prior works \cite{schulmantalk}. 

\label{subsec:synthsetup}
\label{sec:assumptions}

\begin{figure*}[t!]
\vskip 0.2in
\centering
\subfigure[Impact of Finetuning Dataset]{
\centering 
    \includegraphics[width =0.25\textwidth]{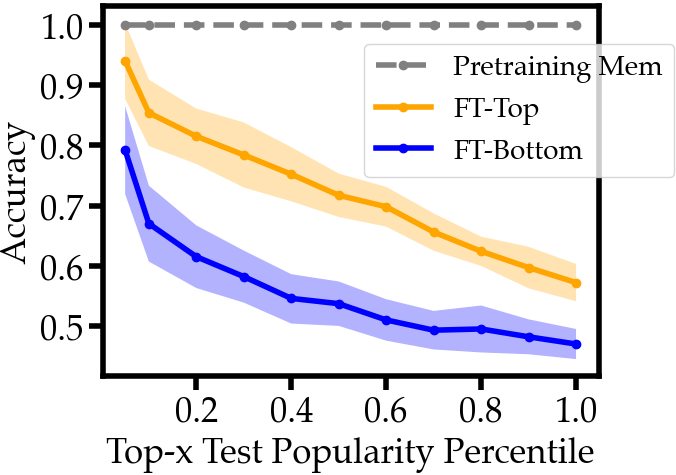}
    \label{subfig:main}
}
\hfill
\subfigure[Effect of Zipf Alpha]{
    \centering
    \includegraphics[width = 0.23\textwidth]{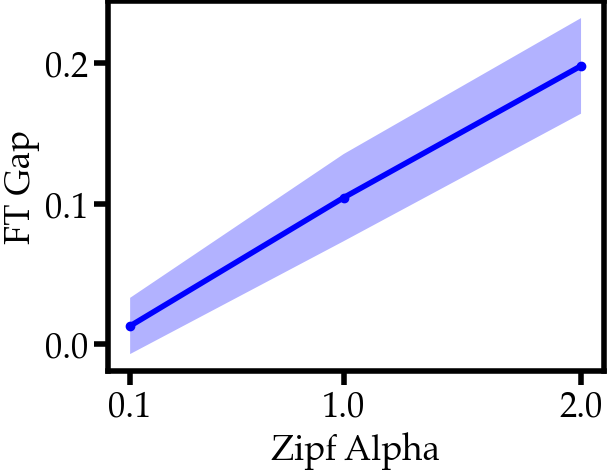}
    \label{subfig:zipfparam}
}
\hfill
\subfigure[Effect of Pretraining Steps]{
\centering
    \includegraphics[width = 0.23\textwidth]{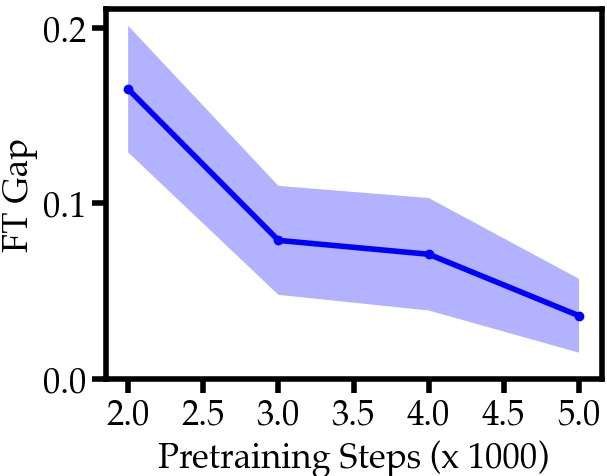}
    \label{subfig:ftsteps}
}
\caption{\textbf{Simulation Study of Finetuning for Knowledge Extraction} (a) We plot the downstream factuality of finetuning on more versus less popular facts, finding that finetuning on more popular facts improves downstream factuality (b) We plot the difference between finetuning on \texttt{FT-Top} and \texttt{FT-Bottom} as a function of the subject Zipf parameter. We find that on increasingly \emph{long-tailed}  datasets, the impact of finetuning dataset is amplified.  (c)  We plot the difference between finetuning on \texttt{FT-Top} and \texttt{FT-Bottom} as a function of pretraining steps, finding that the difference between the finetuning datasets is mitigated with additional training.}
\label{fig:synthetic}
\vskip -0.2in
\end{figure*}

\subsection{Observations in Simulation}

\textbf{Main Finding: Fine-tuning Fact Popularity Impacts Downstream Performance} In Figure \ref{subfig:main}, we plot the accuracy of training on the most popular (\texttt{FT-Top}) and least popular (\texttt{FT-Bottom}) entities in the finetuning dataset. We find that the choice of finetuning dataset \emph{significantly} impacts downstream QA factuality. Concretely, fine-tuning on examples corresponding to the most popular facts in pretraining results in a 10\% improvement in factuality. This difference is amplified as we include relatively less popular data in the test set. For example, the difference between \texttt{FT-Top} and \texttt{FT-Bottom} doubles when we extend our test set from the top $5\%$ to the top $10\%$ most popular entities and persists as we include increasingly unpopular entities. 

\paragraph{Impact of Long-Tailedness in Pretraining Corpus} In Figure \ref{subfig:zipfparam}, we examine the impact of the Zipf $\alpha$ parameter on this phenomena. Intuitively, as $\alpha$ increases, the difference in frequency between more and less popular facts increases. On the other hand, lower $\alpha$ values result in a more uniform distribution over facts. We observe that increasing $\alpha$ exacerbates the differences between the fine-tuning datasets, whereas lowering $\alpha$ largely closes the gap. These findings suggest that differing impacts of the fine-tuning datasets is correlated with how significantly facts differ in their pretraining frequency.

\paragraph{Impact of the Number of Pretraining Steps
} Previously, we observed that the \emph{long-tailedness} of the pretraining distribution controls sensitivity to the fine-tuning dataset. One hypothesis to explain this could be that less frequent facts might not be stored in the model, but we observe that the gap between more and less popular facts is present even when all facts can be extracted in $(s,r)$ form, as evidenced by the pretraining memorization accuracy in Figure \ref{subfig:main}. This suggests that the gap is driven primarily by differences in the internal \emph{fact-storage}. In Figure \ref{subfig:ftsteps}, we investigate this by plotting the gap between \texttt{FT-Top} and \texttt{FT-Bottom} as a function of pretraining steps. We find that with more pretraining steps, the gap decreases, indicating that these internal differences in fact storage disappear as all facts are seen a sufficient number of times. However, achieving this in real large language models would likely be impractical due to the large scale of pretraining data, and the regime of practical interest shows vast difference in downstream performance depending on the finetuning distribution.

\subsection{Conceptual Model: Factual Salience}

Our findings in simulation suggest a \emph{continuous progression} in whether a model ``knows" a particular fact. This controls how well a fact can be extracted in downstream settings, as seen in the decline of downstream accuracy with popularity in Figure \ref{subfig:main}. Moreover, the extent to which the model knows a fact also determines its behavior in finetuning, as evidenced by the gap between \texttt{FT-Top} and \texttt{FT-Bottom} in Figure \ref{subfig:main}. We refer to this intuitive quantity of how well a model knows a fact as the \emph{fact salience} and provide a formal analysis in Section \ref{sec:thy}.

Our simulated results indicate that fact salience is related to the frequency of facts in the pretraining corpus. In particular, we see that differences in the salience are exacerbated as the pretraining distribution becomes more \emph{long-tailed}. However, we also find that these differences are mitigated with additional pretraining, suggesting that they are driven primarily by facts that have been seen only a few times. Importantly, this matches the regime of typical language model training, where roughly \emph{single-epoch} training is performed over a diverse and long-tailed pretraining corpus. In this setting, many facts are likely to be seen only a few times, since multiple passes are not performed.

\section{Theoretical Analysis of Factual Salience}
\label{sec:thy}

In the previous section, we intuitively introduced \emph{fact salience} and hypothesized that it plays a central role in factuality. We now formalize this intuition in a one-layer transformer. We give a quantitative definition of fact salience in this simplified setting (Section \ref{subsec:attimbalance}) and justify its relationship to downstream fact extractability (Theorem \ref{thm:hiddeninfosec4}). Next, we demonstrate that fine-tuning on less salient facts can suppress pretrained knowledge (Theorem \ref{thm:attndynamicssec4}). Finally, we prove that the factual salience increases as a fact is seen during pretraining, justifying the use of pretraining frequency as a proxy (Section \ref{subsec:freqissalience}). In Appendix \ref{app:numerical_experiment_1_layer}, we validate our theory numerically.

\textbf{Simplified Model}
We analyze a one-layer, single-headed transformer model with fixed, orthogonal token embeddings (denoted as $\phi(t)$ for token $t \in \mathcal{T}$). Our model has two learnable parameters: the value matrix, $\valuematrix$, and the key-query matrix $\kq$ and we assume that $\valuematrix, \kq \in \mathbb{R}^{|\mathcal{T}| \times |\mathcal{T}|}$ (i.e. the value, key and query projections preserve the dimensions of the embeddings). We fix the output head to be the identity operation. The forward pass of our model on the input sequence $(s,r)$ can thus be decomposed as follows:
 \begin{equation*}
    X = \begin{bmatrix} \embed{s} & \embed{r} \end{bmatrix}
\end{equation*}
In the first step, the token sequence is embedded and concatenated to an input embedding matrix $X$.
\begin{equation*}
    \selfatt{\valuematrix}{\kq}{X} = \valuematrix X \sigma(X \kq X_{-1}) 
\end{equation*}
Next, the input embedding matrix passes through a single head of self-attention to compute the last-token activation. 
\begin{equation*}
    \ntp{\valuematrix{}}{\kq{}}{s}{r} = \sigma(\selfatt{\valuematrix}{\kq}{X})
 \end{equation*}
Finally, we compute a probability distribution over the \emph{next token} $\ntp{\valuematrix}{\kq}{s}{r}$ as a softmax over the last output of the self-attention layer. An extended analysis of this model is provided in Appendix \ref{app:notations}.

\textbf{Remark.} Since $\valuematrix$ is full rank, the parameterization above is sufficiently expressive to achieve 100\% $\textrm{argmax}$ decoding accuracy (as described in Appendix \ref{app:notations}) on any pretraining dataset where every $(s,r)$ has a unique answer (see Appendix \ref{app:memorizationproof} for proof).

\subsection{Quantifying Factual Salience}
\label{subsec:thyfactsalience}
In Section \ref{sec:synth}, we hypothesized that facts are stored with different strengths in the model weights after pretraining, impacting both their extractability and their behavior in finetuning. In this section, we explicitly quantify this strength in a one-layer transformer.
\begin{definition}[Fact Salience]
    For a fact $(s,r,a)$, we define the fact salience \salience{s}{r}{a} as $\transpose{\embed{a}} \valuematrix{} \embed{s}$.
\end{definition}
Since we fix the model's output transformation to be the identity, $\valuematrix \embed{s}$ can be viewed as encoding an un-normalized probability distribution over the next token, \emph{conditioned only on} $s$. Thus, \salience{s}{r}{a} measures how well the model "stores" the correct answer in relation to $s$. Intuitively, for the fact to be extractable downstream (when prompted by $(s, p_{r})$), the model can only rely on information stored in $s$ because $p_{r}$ is unseen during pretraining. Additionally, we observe that \salience{s}{r}{a} does not depend on the attention parameters as all memorization is implemented by \valuematrix{} (as demonstrated in Appendix \ref{app:memorizationproof}). In the next section, we demonstrate the role played by $\kq$ in modulating the contribution of this stored knowledge to the model's output.
\subsection{Attention Imbalance}
\label{subsec:attimbalance}
In the previous section, we defined a measure of how well knowledge is internally stored in a one-layer transformer. In this section, we study the role of the attention mechanism in \emph{controlling} how this stored knowledge contributes to the model's output. In particular, we show that \emph{imbalances} in the attention scores of $s$ and $p_r$ can suppress pretrained knowledge.

\begin{theorem}[Attention imbalance can lead to hidden knowledge] For pretraining data \dpre, where all $a$ appear at least once, suppose there exists a value matrix \valuematrix{} satisfying mild assumptions \ref{assump:nonuniformlabel} to \ref{assump:allfactmemorized}. Then the one-layer transformer \ntp{\valuematrix}{0}{s}{p_{r}} achieves $100\%$ accuracy under $\argmax$ decoding, but there exists $\kq$ s.t. \ntp{\valuematrix}{\kq}{s}{p_{r}} does not achieve 100\%.
\label{thm:hiddeninfosec4}
\end{theorem}
We provide the specific construction in Appendix \ref{app:hiddeninformationproof} and discuss the relevant assumptions. To summarize, for each relation $r$, we can ensure that a subset of facts with that relation is predicted incorrectly by sending the attention weight on the subject token towards $0$ (equivalently, increasing the attention on the prompting token towards $1$). However, \emph{not all facts are equally susceptible to being suppressed in this way} as we highlight below:\\

\textbf{Fact Salience Controls Robustness to Attention Imbalance} Our proof of Theorem \ref{thm:hiddeninfosec4} relies on ensuring that the attention to the subject token when prompting with $(s,p_{r})$ is sufficiently low. In Appendix \ref{app:hiddeninformationproof}, we demonstrate that an incorrect prediction occurs when the attention to the subject token $\att{s} \leq \frac{d}{\salience{s}{r}{a}}$, for a constant $d$. This formalizes our intuition that fact salience determines how robustly a fact is stored: a small attention imbalance can only force an incorrect response on facts that are less salient.

Next, we make a connection to the phenomena of hidden knowledge, where a LLM outputs an incorrect response despite the correct response being extractable through other probing methods.

\textbf{Remark: Hidden Knowledge}
As Theorem \ref{thm:hiddeninfosec4} does not allow any modification of the value matrix, all factual associations are still stored in $\valuematrix$ and could potentially be extracted by examination of the model's internal parameters. As such, our theory agrees with empirical findings where factually correct knowledge can be extracted from model representations, despite an incorrect generation.

Ultimately, we observe that even when factual knowledge is stored in model parameters, it can be suppressed from the output by attention imbalance. In Section \ref{subsec:supress}, we study the fine-tuning process and demonstrate how attention imbalance can arise.

\subsection{Fine-tuning Attention Dynamics}
\label{subsec:supress}
In Section \ref{subsec:attimbalance}, we showed that imbalances in attention can harm factuality by suppressing stored knowledge. Here, we prove that the \emph{facts seen in finetuning} play an important role in controlling this imbalance. Concretely, finetuning on low-salience facts can exacerbate attention imbalance, while the inclusion of high-salience facts can counteract it. We begin by defining two quantities that appear in the $\kq$ gradient during finetuning (i.e. updating on $(s, p_{r}, a)$ triples).

\begin{definition}[Subject Token Relevance]
$\srel =  \transpose{(\embed{a}- \ntp{\valuematrix}{\kq}{s}{p_{r}})}(\valuematrix{} \embed{s})$
\label{s_token_rel}
\end{definition}
and correspondingly the relation token relevance:\\
\begin{definition}[Relation Token Relevance]
$\prel = \transpose{(\embed{a} - \ntp{\valuematrix}{\kq}{r}{p_{r}})} \valuematrix \embed{p_r}.$
\end{definition}

As derived in the appendix, the update to the attention matrix takes the form 
\begin{equation*}
 -\partialderiv{L}{\kq} \propto (\srel-\prel)(\embed{s}\transpose{\embed{p_{r}}} - \embed{p_{r}}\transpose{\embed{p_{r}}}).
\end{equation*}

The term $\embed{s}\transpose{\embed{p_{r}}}$ increases the attention on $s$, while the term $\embed{p_{r}}\transpose{\embed{p_{r}}}$ increases attention on $p_{r}$. Thus, the $\kq$ gradient up-weights attention on the most relevant token (as measured by $s_\text{rel}$ and $p_\text{rel}$).
\begin{theorem}[Factuality vs. Nonfactuality Inducing Gradients]
When finetuning on a fact $(s,p_{r})$, if $\srel - \prel <0$ then the attention update  $-\partialderiv{L}{\kq}$ decreases the attention on all $s'$ when prompting with $(s', p_{r})$, whereas when $\srel - \prel>0$, $-\partialderiv{L}{\kq}$ increases the attention on all $s'$ when prompting with $(s', p_r)$.
\label{thm:attndynamicssec4}
\end{theorem}
We postpone the formal proof to Appendix \ref{app:att} but provide the following key observations.

\textbf{Role of Factual Salience:} Observe that the definition of subject token relevance (Def. \ref{s_token_rel}) includes the previously defined fact salience. Intuitively, gradient steps taken on less salient facts (relative to the token's correlation with the final output $\ntp{\valuematrix}{\kq}{s}{r}$) downweight attention on the $s$ token (where pretraining knowledge is stored) and push the transformer globally towards attention imbalance. 

\textbf{Global Effect of $p_{r}$ Attention Updates:} The term $\transpose{\embed{p_{r}}} \kq \embed{p_{r}}$ appears in the forward pass of \emph{all} facts with relation $r$. Therefore, updates on a fact where $\srel- \prel < 0$ implicitly decrease attention on all $s \in \subjsect$ (by increasing the attention score on $p_r$). When training on many such $(s, p_{r})$, these updates can accumulate and contribute to significant attention imbalance. Conversely, when $\srel-\prel>0$ the attention on $p_r$ will be decreased, implicitly up-weighting the subject attention for all $s \in \mathcal{S}$. Importantly, this is not a specific consequence of the $(s, p_{r})$ ordering examined in this work: it holds whenever the final prompt token is not subject-specific.

\subsection{Fact Popularity and Salience}
\label{subsec:freqissalience}
While our analysis so far has relied on how strongly facts are internally stored by the model ($\salience{s}{r}{a}$), it is unclear how to compute this quantity beyond the simplified one-layer transformer setting. Here, we verify that the number of times a fact is seen during pretraining correlates with its salience, as suggested by our results in Section \ref{sec:synth}. This suggests pretraining popularity as a proxy for salience. 

\begin{theorem}[Lower bound on fact salience]
Consider pretraining $\ntp{\valuematrix}{\kq}{s}{r}$ on a dataset $\dpre$ of size $N$ for one epoch with learning rate $\epsilon$. Suppose that the $\norm{\kq}_{\infty}<C_{KQ}$ and $\norm{W_{V}}_{\infty}<C_{V}$ throughout training. Suppose that the combination $(s,r)$ appears $n$ times and $s$ appears no more than $n\frac{\exp(-C_{KQ})(|\mathcal{T}|-1)}{2\exp(C_{V})}$ times. Then $\salience{s}{r}{a} \geq nc_{1}\epsilon$ where $c_{1} > 0$.
\label{thm:tokenlearning}

\end{theorem}

We postpone the proof to Appendix \ref{app:tokenlearning}.

Ultimately, our examination of the one-layer transformer provides a tight-fitting conceptual explanation of our simulated observations in Section \ref{sec:synth}. We quantify how strongly a fact is stored in the pretrained weights (fact salience) and demonstrate it grows with pretraining frequency (Theorem \ref{thm:tokenlearning}). Our analysis illustrates that fact salience plays a central role in determining (a) the suppression of pretrained knowledge (Theorem \ref{thm:attndynamicssec4}) and (b) how robustly a fact can be extracted at test time (Theorem \ref{thm:hiddeninfosec4}). We verify this intuition in large language models in Section \ref{sec:real}.

\section{Large Language Model Experiments}
\label{sec:real}
In this section, we verify our findings on the role of the QA dataset when finetuning pretrained large language models (Llama 7B and Mistral-7B). Unlike Section \ref{sec:synth}, where we prescribed an idealized model of the pretraining distribution, the settings examined here test our theory with models trained on large-scale pretraining corpora.

\subsection{Controlled Experiment}
\label{subsec:ctrl}
We first perform a controlled experiment to test the impact of fact salience without confounders.

\renewcommand{\tabcolsep}{2pt}
\begin{table}
\caption{Construction of PopQA-Controlled}
\label{popqa-con}
\centering
\begin{tabular}{l l l }
\toprule
{Question} &   {Ans}   &   {Pop.}  \\
\midrule
\textbf{FT-Top}                                                   &   &   \\
    In what country is Chrysler?  &   USA      &   55586       \\
    What sport does Matty Blair play?          &       Soccer                 &    50643\\
    What is Vanessa Angel's occupation?          &      Actor               &    157667 \\
    \midrule
\textbf{FT-Bottom}                                                 &   &           \\
     In what country is Robinson?  &   USA               &   142         \\
     What sport does Patrick Edema play?           &      Soccer                 &    46  \\
    What is Edwin Wallock's occupation?             &     Actor                &    68  \\
    \bottomrule
\end{tabular}
\end{table}

\textbf{Controlled Setup} 
  To isolate the effect of fact salience on downstream performance, we construct two fine-tuning datasets that differ in fact salience but have the same make-up of relation types and corresponding answers. We use a subset of the PopQA dataset \cite{mallen2023trust} consisting of the \textbf{country}, \textbf{sport} and \textbf{occupation} relations, which we refer to as PopQA-Controlled. We take all questions from each relation with the \emph{most frequent answer} (respectively USA, Soccer, and Actor) and divide them into more and less popular halves (respectively \texttt{FT-Top} and \texttt{FT-Bottom}). Examples from the two fine-tuning datasets are shown in Table \ref{popqa-con}. We disambiguate whether the fine-tuned models learn to use their pretrained knowledge or shortcuts in fine-tuning by testing on questions whose answers are not one of the three seen in fine-tuning. Our theory predicts that fine-tuning on more salient facts would encourage the model to use pretrained knowledge, resulting in better test performance.

\textbf{Results}
In Table \ref{popqa-corr-results}, we observe a significant decrease (7\%) in the factual accuracy of models finetuned on \texttt{FT-Bottom} versus \texttt{FT-Top}. Our results establish that both (a) the varying impact of finetuning on popular versus unpopular knowledge occurs in language models and (b) this effect cannot be explained by correlations between popularity and answer entities or relation types. 

\begin{figure*}[t!]
\vskip 0.2in
\centering
\subfigure[Average Subject Attention Score]{
\centering 
    \includegraphics[width =0.27\textwidth]{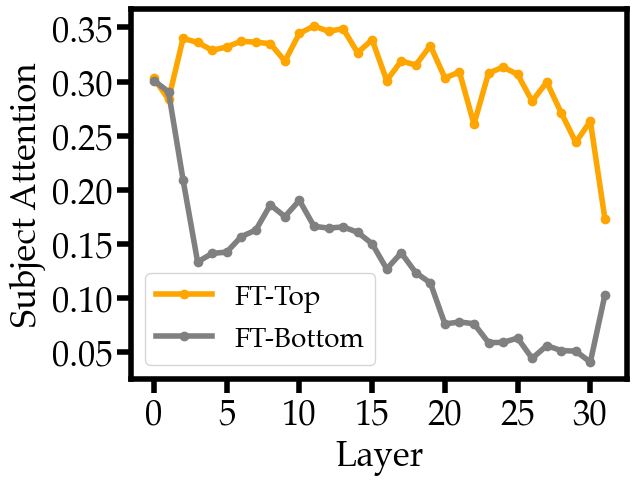}
}
\hfill
\subfigure[Attention Pattern on FT-Top versus FT-Bottom]{
    \centering
    \includegraphics[width = 0.6\textwidth]{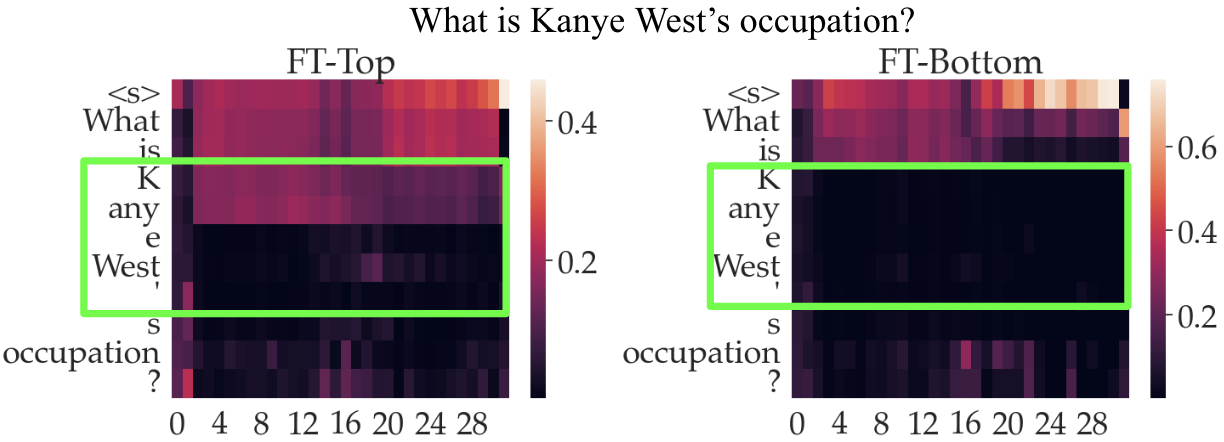}
}
\hfill

\caption{\textbf{Analysis of Llama-7B Attention Patterns} (a) We plot the maximum attention score over subject tokens for Llama-7B models finetuned on \texttt{FT-Top} and \texttt{FT-Bottom} across layers, where the maximum attention score is averaged over the heads in each layer. All results are averaged over examples in the PopQA-Controlled test set. (b) We compare the attention patterns for a specific question between the \texttt{FT-Top} and \texttt{FT-Bottom} fine-tuned models. The tokens corresponding to the subject are enclosed within the green rectangle.}
\label{fig:att}
\vskip -0.2in
\end{figure*}

\begin{table}
\caption{Results on PopQA-Controlled}
\label{popqa-corr-results}
\centering
\begin{tabular}{l l l}
    \toprule
    {Method} &   {Test-Acc} \\
    \midrule                                           
    Zero-Shot Prompting  &  20.1\%      \\
    FT-Top & 44.5 \% \\
    FT-Bottom & 37.4\% \\
    \midrule
\end{tabular}
\end{table}
 
\begin{figure}[t!]
\vskip 0.2in
\centering
\includegraphics[width = 0.3\textwidth]{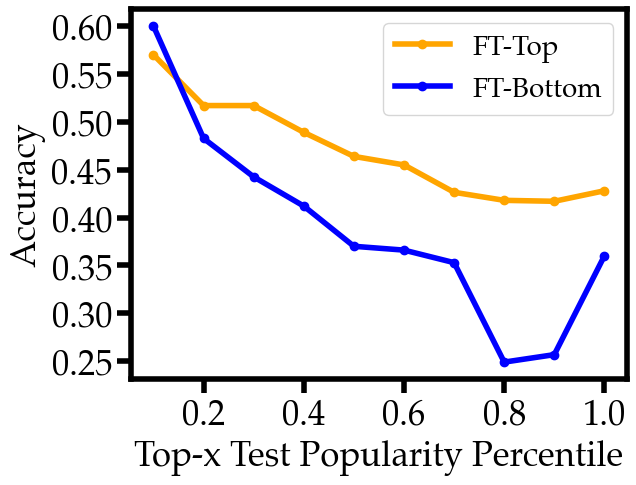}

\caption{\textbf{PopQA-Controlled Test Accuracy on Popularity Percentiles} We plot the accuracy on the top $x$ popularity percentiles of the PopQA-Controlled test set as a function of $x$. We compare the performance of finetuning on \texttt{FT-Top} versus \texttt{FT-Bottom}. We observe that while both finetuning datasets perform comparably on the most popular facts in the test set, training on the less popular data significantly underperforms on relatively less popular test questions.}
\label{fig:popqacorrstratified}
\vskip -0.2in
\end{figure}

\textbf{Stratified Analysis} In Figure \ref{fig:popqacorrstratified}, we additionally observe a surprising trend: the gap between \texttt{FT-Top} and \texttt{FT-Bottom} \emph{increases} as we consider increasingly less popular test set examples. While both finetuning datasets yield similar results in the most-popular 20\% of the test set, the gap between the methods widens as we include increasingly unpopular questions, dropping sharply around the 70th popularity percentile. This finding evinces that our observations \textit{are not} merely a result of matching the finetuning distribution to the test set in popularity (where we would expect large gains for \texttt{FT-Top} on more popular knowledge). Counter-intuitively, skewing the finetuning dataset to more popular examples appears to be especially beneficial in improving performance on less popular knowledge.

\textbf{Analysis of Attention Patterns} 
We study the attention patterns of models fine-tuned on \texttt{FT-Top} versus \texttt{FT-Bottom} and find that they match the theoretical predictions made in Section \ref{sec:thy}. In the left panel of Figure \ref{fig:att} we plot the average attention to the subject tokens (over the test set) as a function of Llama-7B layer index and find that \texttt{FT-Top} trained models attend significantly more to the subject than do models fine-tuned on \texttt{FT-Bottom}. On the right panel, we visualize the attention patterns of the \texttt{FT-Top} and \texttt{FT-Bottom} trained models and see that the attention to the subject-relevant tokens (highlighted in green) is suppressed after training on \texttt{FT-Bottom}. In this setting, these results provide evidence that the mechanistic prediction made in our one-layer transformer model in Section \ref{sec:thy} is reflective of what occurs in a real large language model. Further experimental results are presented in Appendix \ref{app:att}.

\subsection{Real QA Datasets}
Previously, we demonstrated the impact of the fine-tuning QA dataset on question-answering ability in a controlled setting. In this section, we test the implications of our findings for improving factual QA performance.

\subsubsection{Setup}
\textbf{Datasets} We specialize our evaluation to short answer and multiple choice QA involving facts of varying popularity (frequency in the pretraining data). \citet{mallen2023trust} introduce the PopQA dataset which is sampled to include questions about a range of popular and less-popular topics. We also examine a subset of the Entity Questions \cite{sciavolino2022simple} dataset, which includes a diverse range of popular and less popular facts. In both datasets, we utilize the Wikipedia page count of the question \emph{subject-entity} as a proxy for pretraining fact frequency \cite{mallen2023trust, razeghi2022impact}. This is necessary as it is challenging to directly measure fact popularity on large-scale pretraining corpora. Finally, we examine a subset of the MMLU dataset \cite{hendrycks2021measuring} consisting of history questions. Here, we use the pretrained model's loss as a proxy for fact popularity as introduced by \citet{kang2024unfamiliar}. 

\textbf{Models} Our experiments are performed on the Llama 7B \cite{touvron2023llama} and Mistral-7B \cite{jiang2023mistral} pretrained base language models. Restricting to base (non-chat-tuned) models allows us to directly study the effect of pretraining knowledge frequency without confounding introduced by prior finetuning stages. In all experiments, we use the best of LoRA \cite{hu2021lora} and full fine-tuning. In addition, we tune over standard regularization techniques including weight decay, early stopping, and learning rate individually for each model and fine-tuning dataset (further details in Appendix \ref{sec:hparamtunings}). 

\textbf{Evaluation} We evaluate the performance of models on short-answer QA by appropriately normalizing all model predictions and ground-truth answers and checking for exact string matches. We describe the specific normalization techniques in Appendix \ref{app:eval}. We note that both short answer datasets used in this work provide multiple synonymous ground-truth answers, mitigating the potential harshness of exact-match-based evaluations. For multiple-choice evaluation (MMLU), we evaluate the exact match of the model with the ground-truth answer choice.
\begin{table}
\caption{MMLU-History}
\label{table:mmluhist}
\centering
\begin{tabular}{l l l}
    \toprule
    {Finetuning Dataset} &   {Llama-7B} &{Mistral-7B} \\
    \midrule                                           
    FT-Top  &  \textbf{61.4}\% (0.3) &   \textbf{68.7}\% ($0.5$)  \\
    FT-Bottom & 55.6 \% ($0.4$) & 59.4\% ($0.5)$\\
    FT-Whole & 58.8\% $(0.2)$  & 67.4 \%($0.4)$ \\
    \midrule
\end{tabular}
\end{table}

\begin{figure*}[t!]
\vskip 0.2in
\centering
\includegraphics[width=1.0\textwidth]{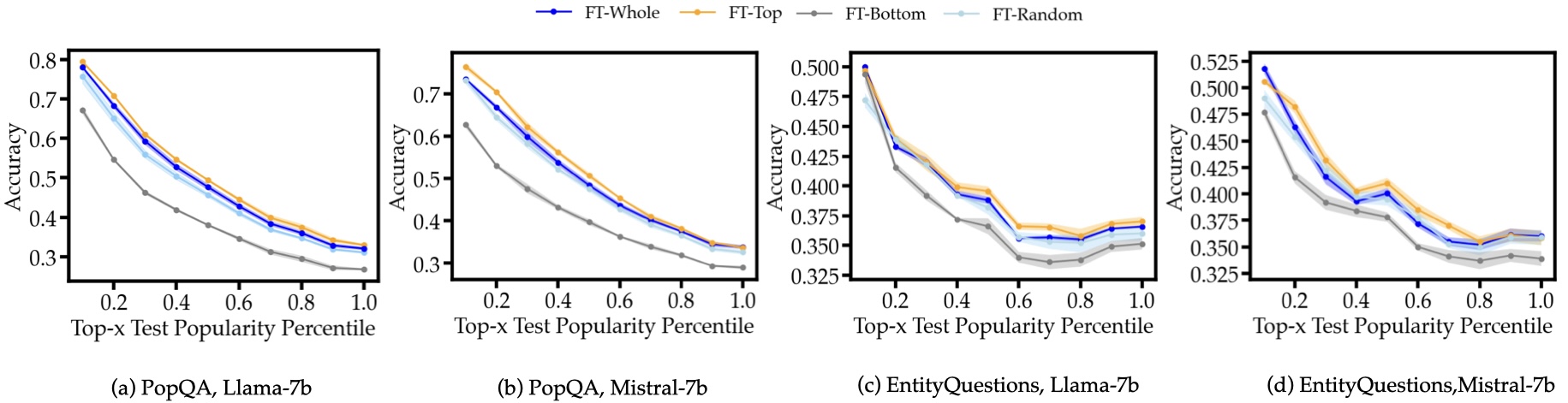}
\caption{\textbf{Finetuning Performance on Real Datasets} We plot the factual QA accuracy across two models and question-answering datasets under different fine-tuning strategies. \texttt{FT-Top} denotes finetuning on the most popular half of data, \texttt{FT-Whole} denotes finetuning on the whole training dataset, \texttt{FT-Random} denotes finetuning on a randomly selected half of the data, and \texttt{FT-Bottom} denotes finetuning on the lower 50\% of the data, sorted by popularity. We plot performance restricting to the top-$x$ popularity percentiles of the test set.}
\label{fig:finetuningreal}
\vskip -0.2in
\end{figure*}

\subsubsection{Results}

\textbf{Unpopular Facts Harm Downstream Factuality} In Figure \ref{fig:finetuningreal}, we observe that finetuning on the least popular knowledge consistently under-performs across both QA datasets (PopQA and Entity-Questions) and models (Llama-7B and Mistral). Similar results are also seen in Table \ref{table:mmluhist} on the MMLU dataset, where finetuning on less confident examples performs significantly (7-10\%) worse than both the top and whole datasets. The consistency of this observation across models and tasks supports that our observations are a general property of the finetuning process, rather than an artifact of a particular LLM or dataset.

\textbf{Impact Relative to Test Popularity} Figure \ref{fig:finetuningreal} displays similar trends relative to test-set popularity as those seen in the more restricted settings. In particular, we observe that the gap between fine-tuning on more versus less popular examples widens away from the most popular 10\% of the test points as we include more unpopular points in our test set. The advantage of training on more popular examples persists even when we include the least popular test points. This finding provides further evidence of our hypothesis that although some highly popular facts are relatively invariant to the choice of fine-tuning dataset, performance on relatively less popular facts varies more significantly.

\textbf{Popular Facts Mitigate Unpopular} Surprisingly, we find that even a \emph{randomly selected} 50\% subset (plotted sky-blue in Figure \ref{fig:finetuningreal}) significantly outperforms \texttt{FT-Bottom}, performing only slightly worse than \texttt{FT-Top} across all settings. This suggests that some popular points (which would be present in \texttt{FT-Random} but not \texttt{FT-Bottom}) can significantly mitigate the damage incurred by finetuning on less popular knowledge. Moreover, this conclusion is supported by our theoretical analysis: the gradients on more popular examples \emph{globally counteract} attention imbalance, as shown in Theorem \ref{thm:attndynamicssec4}.

\textbf{Finetuning Data Quantity in Question-Answering} In Figure \ref{fig:finetuningreal}, we compare the performance of the \emph{best top popularity subset} with fine-tuning on the entire training dataset. Across all settings, we observe that training on a smaller subset of the most popular facts performs comparably or better than using the entire dataset. Moreover, these variations are amplified on the same percentiles as the difference between \texttt{FT-Top} and \texttt{FT-Bottom} (i.e. between the 30th-60th popularity percentiles). In Table \ref{table:mmluhist}, we similarly observe that training only on the most familiar MMLU examples performs better than using the whole dataset across both models. This suggests that (a) \emph{only a subset} of the most popular training points are actually helpful in fine-tuning for factual question-answering and that (b) including the additional QA examples could be harmful to facts that are especially sensitive to finetuning distribution.
\section{Related Works}
    \textbf{Impact of Unfamiliar Knowledge} Recent works have examined the impact of unfamiliar examples during finetuning. \citet{kang2024unfamiliar} argues that unfamiliar examples in the fine-tuning dataset determine how a model responds to unfamiliar test examples. However, they do not consider the impact of unfamiliar fine-tuning examples on the general factuality of the model as is the focus of this work. Concurrently to this work, \citet{gekhman2024does} demonstrate empirically that finetuning on unfamiliar examples can worsen factuality, characterizing it as a result of overfitting. In this paper, however, we present a conceptual model of this phenomena, demonstrating that it arises from suppression of pretrained knowledge in favor of generic ``shortcuts". Our theory additionally explains the varying impact that different fine-tuning strategies have on \emph{test points} of varying popularity or familiarity.
    
    \textbf{Reliable Factuality of Language Models} Prior works have extensively studied challenges and approaches for improving the factuality of LLMs. \citet{mallen2023trust} and \citet{kandpal2023large} demonstrate that LLMs often underperform on obscure or long-tailed knowledge. \citet{li2023inferencetime} find that the factuality of language models can be improved by upweighting certain attention heads. Similarly, \citet{burns2022discovering} demonstrate that unsupervised internal probes can reveal factual associations in language models, even when they ultimately output an incorrect response. \citet{chuang2023dola} demonstrate that contrasting the final prediction from earlier layers of language models can improve factual accuracy. Prior works have also examined methods to improve factuality via abstention.\citet{varshney2023stitch} demonstrate that low confidence outputs can be hallucinations. Similarly, \citet{yuksekgonul2023attention} use token attention scores to detect when language models hallucinate. On the other hand, \citet{yang2023alignment,zhang2023rtuning, schulmantalk} introduce fine-tuning techniques to induce large-language models to refuse questions that are outside their knowledge boundary.  Collectively, these prior works demonstrate failure modes of factual reliability in language models, at times even when they can output the correct answer. In this work, on the other hand, we study the impact of the fine-tuning distribution on the model's downstream factuality.

\textbf{Understanding LLM Mechanisms and Training Dynamics} 
Many prior works have sought to explain the behaviors of language models and understand their failure modes. \citet{allenzhu2023physics} examine the conditions on pretraining data necessary for facts to be stored in an extractable form on a synthetic dataset. \citet{geva2023dissecting} identify the mechanisms by which facts are stored and extracted in language models. \citet{li2023transformers} study one-layer transformer pretraining dynamics on a topic modeling task. \citet{chen-sudden-2023} empirically studies the pretraining dynamics of syntax acquisition in masked language models.  \citet{tian2023scan} analyze the attention dynamics of one-layer transformers, demonstrating that uniquely co-occurring tokens are upweighted in attention. \citet{liu2023exposing} examine long-range reasoning failures of large language models, attributing them to erroneous attention scores. In this work, we focus on understanding the mechanics of fine-tuning relating to promoting or suppressing pretrained knowledge, thereby impacting the extractability of facts downstream.

\section{Discussion}

In this work, we investigate the impact of QA dataset composition on model factuality, making a notable finding: fine-tuning on questions about well-known facts uniformly improves factuality over fine-tuning on less known facts. We observe this trend across a range of simulation and real-world settings and develop a conceptual model of QA finetuning in a simplified one-layer transformer. Our results challenge intuitive heuristics for designing QA fine-tuning datasets. In particular, over-representing well-known facts in QA fine-tuning can actually be beneficial. Our results can inform principled methods to improve the downstream factuality of language models. Guided by our theory, a valuable area for future work can be developing regularization techniques to mitigate attention imbalance during finetuning. Another promising avenue is curriculum learning, which could enable more obscure facts to be trained on \emph{after} finetuning on more popular knowledge to mitigate attention imbalance. Finally, we hypothesize that our conceptual model can guide the development of synthetic data to efficiently improve knowledge extractability.




\section*{Acknowledgements}
This research was supported by the Center for AI
Safety Compute Cluster. Any opinions, findings,
conclusions, or recommendations expressed in
this material are those of the author(s) and do not
necessarily reflect the views of the sponsors.
This work was supported in part by the AI2050
program at Schmidt Sciences (Grant \#G2264481).
We gratefully acknowledge the support of Apple. TH acknowledges the support of Open Philanthropy.



\section*{Impact Statement}
This paper presents work whose goal is to advance the field of Machine Learning. There are many potential societal consequences of our work, none of which we feel must be specifically highlighted here.

\bibliography{example_paper}

\begin{thebibliography}{35}
\providecommand{\natexlab}[1]{#1}
\providecommand{\url}[1]{\texttt{#1}}
\expandafter\ifx\csname urlstyle\endcsname\relax
  \providecommand{\doi}[1]{doi: #1}\else
  \providecommand{\doi}{doi: \begingroup \urlstyle{rm}\Url}\fi

\bibitem[Allen-Zhu \& Li(2023)Allen-Zhu and Li]{allenzhu2023physics}
Allen-Zhu, Z. and Li, Y.
\newblock Physics of language models: Part 3.1, knowledge storage and extraction, 2023.

\bibitem[Burns et~al.(2022)Burns, Ye, Klein, and Steinhardt]{burns2022discovering}
Burns, C., Ye, H., Klein, D., and Steinhardt, J.
\newblock Discovering latent knowledge in language models without supervision, 2022.

\bibitem[Chen et~al.(2024)Chen, Shwartz-Ziv, Cho, Leavitt, and Saphra]{chen-sudden-2023}
Chen, A., Shwartz-Ziv, R., Cho, K., Leavitt, M.~L., and Saphra, N.
\newblock Sudden drops in the loss: Syntax acquisition, phase transitions, and simplicity bias in {MLM}s.
\newblock In \emph{The Twelfth International Conference on Learning Representations}, 2024.
\newblock URL \url{https://openreview.net/forum?id=MO5PiKHELW}.

\bibitem[Chuang et~al.(2023)Chuang, Xie, Luo, Kim, Glass, and He]{chuang2023dola}
Chuang, Y.-S., Xie, Y., Luo, H., Kim, Y., Glass, J., and He, P.
\newblock Dola: Decoding by contrasting layers improves factuality in large language models, 2023.

\bibitem[Elsahar et~al.(2018)Elsahar, Vougiouklis, Remaci, Gravier, Hare, Laforest, and Simperl]{elsahar2018t}
Elsahar, H., Vougiouklis, P., Remaci, A., Gravier, C., Hare, J., Laforest, F., and Simperl, E.
\newblock T-rex: A large scale alignment of natural language with knowledge base triples.
\newblock In \emph{Proceedings of the Eleventh International Conference on Language Resources and Evaluation (LREC 2018)}, 2018.

\bibitem[Gekhman et~al.(2024)Gekhman, Yona, Aharoni, Eyal, Feder, Reichart, and Herzig]{gekhman2024does}
Gekhman, Z., Yona, G., Aharoni, R., Eyal, M., Feder, A., Reichart, R., and Herzig, J.
\newblock Does fine-tuning llms on new knowledge encourage hallucinations?, 2024.

\bibitem[Geva et~al.(2023)Geva, Bastings, Filippova, and Globerson]{geva2023dissecting}
Geva, M., Bastings, J., Filippova, K., and Globerson, A.
\newblock Dissecting recall of factual associations in auto-regressive language models, 2023.

\bibitem[Hendrycks et~al.(2021)Hendrycks, Burns, Basart, Zou, Mazeika, Song, and Steinhardt]{hendrycks2021measuring}
Hendrycks, D., Burns, C., Basart, S., Zou, A., Mazeika, M., Song, D., and Steinhardt, J.
\newblock Measuring massive multitask language understanding, 2021.

\bibitem[Hu et~al.(2021)Hu, Shen, Wallis, Allen-Zhu, Li, Wang, Wang, and Chen]{hu2021lora}
Hu, E.~J., Shen, Y., Wallis, P., Allen-Zhu, Z., Li, Y., Wang, S., Wang, L., and Chen, W.
\newblock Lora: Low-rank adaptation of large language models, 2021.

\bibitem[Huang et~al.(2023)Huang, Yu, Ma, Zhong, Feng, Wang, Chen, Peng, Feng, Qin, and Liu]{huang2023survey}
Huang, L., Yu, W., Ma, W., Zhong, W., Feng, Z., Wang, H., Chen, Q., Peng, W., Feng, X., Qin, B., and Liu, T.
\newblock A survey on hallucination in large language models: Principles, taxonomy, challenges, and open questions, 2023.

\bibitem[Jiang et~al.(2023)Jiang, Sablayrolles, Mensch, Bamford, Chaplot, de~las Casas, Bressand, Lengyel, Lample, Saulnier, Lavaud, Lachaux, Stock, Scao, Lavril, Wang, Lacroix, and Sayed]{jiang2023mistral}
Jiang, A.~Q., Sablayrolles, A., Mensch, A., Bamford, C., Chaplot, D.~S., de~las Casas, D., Bressand, F., Lengyel, G., Lample, G., Saulnier, L., Lavaud, L.~R., Lachaux, M.-A., Stock, P., Scao, T.~L., Lavril, T., Wang, T., Lacroix, T., and Sayed, W.~E.
\newblock Mistral 7b, 2023.

\bibitem[Jiang et~al.(2020)Jiang, Xu, Araki, and Neubig]{jiang2020know}
Jiang, Z., Xu, F.~F., Araki, J., and Neubig, G.
\newblock How can we know what language models know?, 2020.

\bibitem[Joshi et~al.(2023)Joshi, Rando, Saparov, Kim, and He]{joshi2023personas}
Joshi, N., Rando, J., Saparov, A., Kim, N., and He, H.
\newblock Personas as a way to model truthfulness in language models, 2023.

\bibitem[Kandpal et~al.(2023)Kandpal, Deng, Roberts, Wallace, and Raffel]{kandpal2023large}
Kandpal, N., Deng, H., Roberts, A., Wallace, E., and Raffel, C.
\newblock Large language models struggle to learn long-tail knowledge, 2023.

\bibitem[Kang et~al.(2024)Kang, Wallace, Tomlin, Kumar, and Levine]{kang2024unfamiliar}
Kang, K., Wallace, E., Tomlin, C., Kumar, A., and Levine, S.
\newblock Unfamiliar finetuning examples control how language models hallucinate, 2024.

\bibitem[Kazemi et~al.(2023)Kazemi, Mittal, and Ramachandran]{kazemi2023understanding}
Kazemi, M., Mittal, S., and Ramachandran, D.
\newblock Understanding finetuning for factual knowledge extraction from language models, 2023.

\bibitem[Li et~al.(2023{\natexlab{a}})Li, Patel, Viégas, Pfister, and Wattenberg]{li2023inferencetime}
Li, K., Patel, O., Viégas, F., Pfister, H., and Wattenberg, M.
\newblock Inference-time intervention: Eliciting truthful answers from a language model, 2023{\natexlab{a}}.

\bibitem[Li et~al.(2023{\natexlab{b}})Li, Li, and Risteski]{li2023transformers}
Li, Y., Li, Y., and Risteski, A.
\newblock How do transformers learn topic structure: Towards a mechanistic understanding, 2023{\natexlab{b}}.

\bibitem[Liu et~al.(2023{\natexlab{a}})Liu, Ash, Goel, Krishnamurthy, and Zhang]{liu2023exposing}
Liu, B., Ash, J.~T., Goel, S., Krishnamurthy, A., and Zhang, C.
\newblock Exposing attention glitches with flip-flop language modeling, 2023{\natexlab{a}}.

\bibitem[Liu et~al.(2023{\natexlab{b}})Liu, Casper, Hadfield-Menell, and Andreas]{liu2023cognitive}
Liu, K., Casper, S., Hadfield-Menell, D., and Andreas, J.
\newblock Cognitive dissonance: Why do language model outputs disagree with internal representations of truthfulness?, 2023{\natexlab{b}}.

\bibitem[Mallen et~al.(2023)Mallen, Asai, Zhong, Das, Khashabi, and Hajishirzi]{mallen2023trust}
Mallen, A., Asai, A., Zhong, V., Das, R., Khashabi, D., and Hajishirzi, H.
\newblock When not to trust language models: Investigating effectiveness of parametric and non-parametric memories, 2023.

\bibitem[Ouyang et~al.(2022)Ouyang, Wu, Jiang, Almeida, Wainwright, Mishkin, Zhang, Agarwal, Slama, Ray, Schulman, Hilton, Kelton, Miller, Simens, Askell, Welinder, Christiano, Leike, and Lowe]{ouyang2022training}
Ouyang, L., Wu, J., Jiang, X., Almeida, D., Wainwright, C.~L., Mishkin, P., Zhang, C., Agarwal, S., Slama, K., Ray, A., Schulman, J., Hilton, J., Kelton, F., Miller, L., Simens, M., Askell, A., Welinder, P., Christiano, P., Leike, J., and Lowe, R.
\newblock Training language models to follow instructions with human feedback, 2022.

\bibitem[Petroni et~al.(2019)Petroni, Rockt{\"a}schel, Lewis, Bakhtin, Wu, Miller, and Riedel]{petroni2019language}
Petroni, F., Rockt{\"a}schel, T., Lewis, P., Bakhtin, A., Wu, Y., Miller, A.~H., and Riedel, S.
\newblock Language models as knowledge bases?
\newblock \emph{arXiv preprint arXiv:1909.01066}, 2019.

\bibitem[Radford et~al.(2019)Radford, Wu, Child, Luan, Amodei, Sutskever, et~al.]{radford2019language}
Radford, A., Wu, J., Child, R., Luan, D., Amodei, D., Sutskever, I., et~al.
\newblock Language models are unsupervised multitask learners.
\newblock \emph{OpenAI blog}, 1\penalty0 (8):\penalty0 9, 2019.

\bibitem[Razeghi et~al.(2022)Razeghi, au2, Gardner, and Singh]{razeghi2022impact}
Razeghi, Y., au2, R. L. L.~I., Gardner, M., and Singh, S.
\newblock Impact of pretraining term frequencies on few-shot reasoning, 2022.

\bibitem[Roberts et~al.(2020)Roberts, Raffel, and Shazeer]{roberts-etal-2020-much}
Roberts, A., Raffel, C., and Shazeer, N.
\newblock How much knowledge can you pack into the parameters of a language model?
\newblock In Webber, B., Cohn, T., He, Y., and Liu, Y. (eds.), \emph{Proceedings of the 2020 Conference on Empirical Methods in Natural Language Processing (EMNLP)}, pp.\  5418--5426, Online, November 2020. Association for Computational Linguistics.
\newblock \doi{10.18653/v1/2020.emnlp-main.437}.
\newblock URL \url{https://aclanthology.org/2020.emnlp-main.437}.

\bibitem[Schulman(2023)]{schulmantalk}
Schulman, J.
\newblock Reinforcement learning from human feedback: Progress and challenges.
\newblock Talk given at the University of California, Berkeley on April 19, 2023., 2023.
\newblock URL \url{https://www.youtube.com/watch?v=hhiLw5Q_UFg}.

\bibitem[Sciavolino et~al.(2022)Sciavolino, Zhong, Lee, and Chen]{sciavolino2022simple}
Sciavolino, C., Zhong, Z., Lee, J., and Chen, D.
\newblock Simple entity-centric questions challenge dense retrievers, 2022.

\bibitem[Tian et~al.(2023{\natexlab{a}})Tian, Mitchell, Yao, Manning, and Finn]{tian2023finetuning}
Tian, K., Mitchell, E., Yao, H., Manning, C.~D., and Finn, C.
\newblock Fine-tuning language models for factuality, 2023{\natexlab{a}}.

\bibitem[Tian et~al.(2023{\natexlab{b}})Tian, Wang, Chen, and Du]{tian2023scan}
Tian, Y., Wang, Y., Chen, B., and Du, S.
\newblock Scan and snap: Understanding training dynamics and token composition in 1-layer transformer, 2023{\natexlab{b}}.

\bibitem[Touvron et~al.(2023)Touvron, Lavril, Izacard, Martinet, Lachaux, Lacroix, Rozière, Goyal, Hambro, Azhar, Rodriguez, Joulin, Grave, and Lample]{touvron2023llama}
Touvron, H., Lavril, T., Izacard, G., Martinet, X., Lachaux, M.-A., Lacroix, T., Rozière, B., Goyal, N., Hambro, E., Azhar, F., Rodriguez, A., Joulin, A., Grave, E., and Lample, G.
\newblock Llama: Open and efficient foundation language models, 2023.

\bibitem[Varshney et~al.(2023)Varshney, Yao, Zhang, Chen, and Yu]{varshney2023stitch}
Varshney, N., Yao, W., Zhang, H., Chen, J., and Yu, D.
\newblock A stitch in time saves nine: Detecting and mitigating hallucinations of llms by validating low-confidence generation, 2023.

\bibitem[Yang et~al.(2023)Yang, Chern, Qiu, Neubig, and Liu]{yang2023alignment}
Yang, Y., Chern, E., Qiu, X., Neubig, G., and Liu, P.
\newblock Alignment for honesty, 2023.

\bibitem[Yuksekgonul et~al.(2023)Yuksekgonul, Chandrasekaran, Jones, Gunasekar, Naik, Palangi, Kamar, and Nushi]{yuksekgonul2023attention}
Yuksekgonul, M., Chandrasekaran, V., Jones, E., Gunasekar, S., Naik, R., Palangi, H., Kamar, E., and Nushi, B.
\newblock Attention satisfies: A constraint-satisfaction lens on factual errors of language models, 2023.

\bibitem[Zhang et~al.(2023)Zhang, Diao, Lin, Fung, Lian, Wang, Chen, Ji, and Zhang]{zhang2023rtuning}
Zhang, H., Diao, S., Lin, Y., Fung, Y.~R., Lian, Q., Wang, X., Chen, Y., Ji, H., and Zhang, T.
\newblock R-tuning: Teaching large language models to refuse unknown questions, 2023.

\end{thebibliography}
\bibliographystyle{icml2024}

\newpage
\appendix
\onecolumn
\section{Theory Appendix}
\subsection{Notations and Setup}

\label{app:notations}

\paragraph{Representation of Tokens} We consider a synthetic language with a token set $\mathcal{T}$ where $|\mathcal{T}|$. When representing tokens, we consider them to be integers in the interval $[0, |\mathcal{T}|]$. To represent factual associations, we further partition $\mathcal{T} = \mathcal{S} \cup \mathcal{R} \cup \mathcal{A} \cup \{p_{r} |r \in \mathcal{R}\}$. As such $\mathcal{S}$, $\mathcal{R}$, $\mathcal{A}$, and $\{p_{r} |r \in \mathcal{R}\}$ are sets of integers, representing the underlying tokens. 

\paragraph{Embedding Layer} As introduced in Section \ref{sec:thy}, we consider fully fixed, fully orthogonal token embeddings (i.e. the embedding matrix is the identity matrix) and the embedding of a token $i$ is $e_{i} \in \mathbbm{R}^{|\mathcal{T}|}$ (i.e. a vector with all entries $0$ except for the $i$-th component). Moreover, the embedding and unembedding modules are considered to be weight-tied as examined in \citet{li2023transformers}. In this setting, we have that the embedding of a token $i$ is the $i$-th basis vector (i.e. $e_{i}$), and as a result the embeddings of the different tokens are orthogonal to one another. In addition, the $i$-th component of the model output $f(s,r,\valuematrix{}, \kq{})_{i}$ is the probability of token $i$ being the next token (as we discuss further when introducing $\argmax$ decoding).

\paragraph{One-Layer Transformer Architecture}
We consider a one-headed, one-layer transformer in this work with fully orthogonal and weight-tied embedding and unembedding layers. We assume that the key, query, and value matrices are square, thereby preserving the dimensions of the embedding. We additionally assume that the language modeling head corresponds to an identity transformation (this is possible due to the projections preserving the dimensions of the embedding).

Denote a matrix of embedded inputs $X \in \mathbb{R}^{|\mathcal{T}| \times l}$, where $l$ is the sequence length. We can then write the output of this single head of attention (given parameters $W_{K}, W_{Q}$, $W_{V}$) as :
\begin{equation*}
\selfattext{W_Q}{W_K}{\valuematrix}{X} = (\valuematrix X) \softmax{(\keymatrix X)^{T} (\querymatrix X)}
\end{equation*} 
where $\sigma$ denotes the column-wise softmax operation. 

In our simplified model, we consider only one self-attention layer and consider that the language modeling head is the identity, which is possible because the embeddings, query, key, and value projections all lie in  $\mathbb{R}^{|\tokenset|}$. We can then write the \emph{next-token prediction} function, given a sequence of tokens $t_{1},...,t_{l}$, as $f: \mathcal{T}^{l} \rightarrow \Delta(\mathcal{T})$ where $l$ is the sequence length and $\Delta(\mathcal{T})$ is the space of probability distributions over the token space $\mathcal{T}$. Applying our simplifying assumptions, we can write that 
\begin{equation*}
\fextended{[t_{1},,,.t_{l}]} = \softmax{\selfattext{\querymatrix}{\keymatrix}{\valuematrix}{X}_{:-1}}
\end{equation*}
where the subscript $:-1$ denotes the last column of the matrix. Thus, we take the softmax of the (post-self-attention) embedding of the last input token to predict the next token. Note that this can be rewritten:

\begin{equation*}
     \fextended{[t_{1},,,.t_{l}]}  = \softmax{(W^{V} X) \sigma((W^{K} X)^{T} W^{Q} X_{:-1})}.
\end{equation*}.

Note that the actual computation depends only on the product $\transpose{(\keymatrix)} \querymatrix$ and thus, we will reparameterize as $\kq = \transpose{(\keymatrix)} \querymatrix$. For convenience, in the main text, we redefine the $\text{Self-Att}$ function to map from an input sequence embedding matrix to the last token's embedding (i.e. $\text{Self-Att}: \mathbb{R}^{|\mathcal{T}| \times l} \rightarrow \mathbb{R}^{|\mathcal{T}|}$) and we parametrize with only $W^V$ and $W^{KQ}$. Thus, we use the definition

\begin{equation*}
    \selfatt{\valuematrix}{\kq}{X} = (W^{V} X)\sigma(X^{T} W^{KQ} X_{:-1}).
\end{equation*}

For most of our analysis, we will focus on the specialized setting of next-token prediction given the context $(s,r)$ or $(s, p_{r})$. In this specialized setting (considering $(s,r)$ for instance), we have that $X = \begin{bmatrix} \embed{s} & \embed{r} \end{bmatrix}$. We can write the following
\begin{equation*}
\ntp{\valuematrix}{\kq}{s}{r} = \softmax{W^{V} \begin{bmatrix} \embed{s} & \embed{r} \end{bmatrix} \sigma(\begin{bmatrix} \transpose{\embed{s}} \\ \transpose{\embed{r}} \end{bmatrix} \kq \embed{r}}.
\end{equation*}

Finally, we will rewrite this in order to more clearly demonstrate the contribution of different tokens to the final prediction. 

\begin{equation*}
\softmax{\begin{bmatrix} \transpose{\embed{s}} \\ \transpose{\embed{r}}\end{bmatrix} \kq \embed{r}}_{0} \valuematrix \embed{s}  + \softmax{\begin{bmatrix} \transpose{\embed{s}} \\ \transpose{\embed{r}}\end{bmatrix} \kq \embed{r}}_{1} W^{V} \embed{r}.
\end{equation*}

In the following proofs, we will often abbreviate $\Att{s} = \softmax{\begin{bmatrix} \transpose{\embed{s}} \\ \transpose{\embed{r}}\end{bmatrix} \kq \embed{r}}_{0}$ and $\Att{r} = \softmax{\begin{bmatrix} \transpose{\embed{s}} \\ \transpose{\embed{r}}\end{bmatrix} \kq \embed{r}}_{1}$. Thus, the pre-softmax output of the one-layer transformer can be written as 
\begin{equation*}
  \Att{s} \valuematrix \embed{s}  + \Att{r} W^{V} \embed{r}.  
\end{equation*}

\label{sec:datagen}

\paragraph{$\argmax$ Decoding} Recall that the language model embedding layer is is fixed to be an identity matrix. In addition, the output projection is defined to be an identity transformation. Thus, the final token output of the self-attention layer can be interpreted as an un-normalized probability distribution over the token space $\mathcal{T}$. We then define $\argmax$-decoding as: 
\begin{equation*}
\argmax_{t \in \mathcal{T}} (\ntp{\valuematrix}{\kq}{s}{r})_{t}.
\end{equation*}
That is, the index of the maximum value in the one-layer transformer output corresponds to the index of the next token. Additionally, due to the softmax layer preserving the ordering of the vector's components, this is equivalent to:
\begin{equation*}
\argmax_{t \in \mathcal{T}} (\selfatt{\valuematrix}{\kq}{X})_{t}.
\end{equation*}

\subsection{One-Layer Transformer Can Memorize All Facts}
\label{app:memorizationproof}

In this section, we prove that despite its simplified nature, a one-layer transformer is capable of memorizing all facts in the pretraining dataset. This helps establish that our simplified model does not restrict the ability to learn facts in our setting.

\begin{theorem}[One-layer transformer can fully memorize the pretraining dataset] Consider any pretraining dataset $\dpre = \{(s^{(1)}, r^{(1)}, a^{(1)}),..., (s^{(N)}, r^{(N)}, a^{(N)}\}$ such that any $(s,r)$ combination appears exactly once. There exists a one-layer transformer with parameters $\ntp{\valuematrix}{\kq}{s}{r}$ s.t. $\argmax \ntp{\valuematrix}{\kq}{s}{r} = a$ $\forall (s,r,a) \in \dpre$.
\label{thm:onelayermem}
\end{theorem}
\begin{proof}
We will prove that the choice of parameters: $\valuematrix = \sum_{(s,r,a) \in D_\text{pre}} \embed{a} \transpose{\embed{s}} + \embed{a} \transpose{\embed{r}}$ and $\kq = 0$ satisfies the result of the theorem. 

Observe that $\ntp{\valuematrix}{0}{s}{r}  = \softmax{\frac{1}{2} \valuematrix \embed{s^{*}} + \frac{1}{2} \valuematrix \embed{r^{*}}}$. Next, we will define two sets \ansset{s^{*}} and \ansset{r^{*}}
\begin{equation*}
\begin{split}
\ansset{s^{*}} = \{a | (s^{*}, r, a) \in \dpre \, \forall r \in \mathcal{R}\} \\
\ansset{r^{*}} = \{a | (s, r^{*}, a) \in \dpre \, \forall s \in \mathcal{S}\}
\end{split}
\end{equation*}

Intuitively, $\ansset{s^{*}}$ denotes the set of all answers that are observed associated with $s^{*}$ (as $r$ is allowed to vary), and likewise $\ansset{s^{*}}$ is the set of answers seen with $r^{*}$ as $s$ is allowed to vary. Using these sets and applying the orthogonality of the embeddings, we can simplify the expression for $f$ to  
\begin{equation*} 
\ntp{\valuematrix}{\kq}{s}{r} = \sigma(\frac{1}{2} \sum_{a \in A^{s^{*}}} \mathbbm{1}\{a\}  + \frac{1}{2} \sum_{a\in\mathcal{A}^{r^{*}}} \mathbbm{1}\{a\}).
\end{equation*}
We can further simplify by pulling out the term corresponding to $a \in \mathcal{A}^{s^{*}} \cap \mathcal{A}^{r{*}}$. We will abbreviate $\mathcal{A}^\cap = \mathcal{A}^{s^{*}} \cap \mathcal{A}^{r{*}}$,  $\mathcal{A}^{r^*}_{\setminus} = \mathcal{A}^{r^*} \setminus \mathcal{A}^\cap$, and $\mathcal{A}^{s^*}_\setminus = \mathcal{A}^{s^*} \setminus \mathcal{A}^\cap$
\begin{equation*}
\ntp{\valuematrix}{\kq
}{s}{r} = \sigma(\sum_{a \in \mathcal{A}^{\cap}}\embed{a} + \frac{1}{2}(\sum_{a \in \mathcal{A}^{s^{*}}_\setminus} \embed{a} + \sum_{a \in \mathcal{A}^{r^{*}}_\setminus} \embed{a})).
\end{equation*}

For the remainder of the proof, we will examine the \emph{pre-softmax} output of the one-layer transformer, which we will write as
\begin{equation*}
Z = \sum_{a \in \mathcal{A}^{\cap}}\embed{a} + \frac{1}{2}(\sum_{a \in \mathcal{A}^{s^{*}}_\setminus} \embed{a} + \sum_{a \in \mathcal{A}^{r^{*}}_\setminus} \embed{a}).
\end{equation*}

Observe that each $(s,r)$ tuple is associated with only one $a$ in the pretraining dataset $D_{pre}$, implying that $\mathcal{A}^{r} \cap \mathcal{A}^{s} = \{a^{*}\}$ where we define  $a^{*}$ such that $(s^{*}, r^{*}, a^{*}) \in D_\text{pre}$. This implies that (using orthogonality) that:
\begin{equation*}
\transpose{\embed{a^*}} Z = 1.
\end{equation*}
On the other hand, we have
\begin{equation*}
\transpose{\embed{a}} Z = \frac{1}{2} \,\,\,\, \forall a \in \mathcal{A}^{r^*}_{\setminus} \cup \mathcal{A}^{r^*}_{\setminus}.
\end{equation*}
Finally, by orthogonality, we have for all $a \in \mathcal{A} \setminus (\mathcal{A}^{s^{*}} \cup \mathcal{A}^{r^{*}} \cup \mathcal{A}^\cap)$
\begin{equation*}
\transpose{\embed{a}}(\valuematrix \embed{s} + \valuematrix\embed{r}) = 0. 
\end{equation*}

Therefore $\text{argmax}_{\tilde{a}} \transpose{\embed{\tilde{a}}}(\frac{1}{2} \sum_{a \in A^{s}} \embed{a} + \frac{1}{2} \sum_{a \in A^{r}} \embed{a}) = a^{*}$ and we have completed the proof since this holds for arbitrary $(s, r)$.
\end {proof}

Intuitively, the full-rank nature of the embedding and value matrices enable us to use the value matrix as a key-value store which encodes the mapping $(s,r) \rightarrow a$.

\subsection{Proof of Theorem \ref{thm:hiddeninfo}}
\label{app:hiddeninformationproof}
In this theorem, we demonstrate that despite potentially having all factual associations encoded in the value matrix $W^{V}$, the attention weights can be modified such that information is suppressed from the output layer. Our construction here depends primarily on the attention scores becoming imbalanced against the subject token, which we will demonstrate occurs during fine-tuning in the next section. We introduce three additional assumptions regarding the structure of the value matrix and the fact distribution. These assumptions are largely minor in nature: Assumption \ref{assump:nonuniformlabel} simply requires that the activations $\transpose{\embed{a}} W^{V} \embed{s}$ are not all identical. In a similar vein, we assume that each answer is seen at least once in the dataset. Finally, as indicated in the hypothesis of the theorem, we assume that all the facts are memorized. Note that Assumption \ref{assump:allfactmemorized} is a significantly weaker consequence of the one-layer transformer achieving 100\% accuracy -- however it is all that is needed for the proof.

\begin{assumption}[Non-Uniform Relation Marginal]
    $\forall r \max_{a \in \mathcal{A}^{r}} \transpose{\embed{a}}\valuematrix \embed{p_{r}} -\min_{a \in \mathcal{A}^{r}} \transpose{\embed{a}} \valuematrix \embed{p_{r}} > 0$
    \label{assump:nonuniformlabel}
\end{assumption}

\begin{assumption}[Answer Diversity]
    $\forall r \,\, \forall a \in A^{r} \, \exists s \in \mathcal{S} \,\, \text{such that} (s,r,a) \in D_\text{pre}$
    \label{assump:ansdiversity}
\end{assumption}

\begin{assumption}[All Facts Memorized]
$\forall s \in \mathcal{S}$ $\forall r \in \mathcal{R}$  $\transpose{\embed{a}} W^{V} \embed{s} \geq \max_{a' \in \mathcal{A}^{r}} \transpose{\embed{a'}} \valuematrix \embed{r} -  \transpose{\embed{a}} W^{V} \embed{r}$
\label{assump:allfactmemorized}
\end{assumption}

For simplicity, we will additionally assume that all entries of the value matrix $W^{V}$ are greater than or equal to 0. This can be achieved by simply shifting all entries in the matrix, without changing the relative orderings of the activations.
\begin{theorem}[Attention imbalance can lead to hidden information] Consider any value matrix $\valuematrix$ satisfying assumptions \ref{assump:nonuniformlabel} throug \ref{assump:allfactmemorized}. Then a one-layer transformer with parameters $[W^{V},0]$ achieves 100\% accuracy but there exists $W^{QK}$ s.t. $f_{[\valuematrix,\kq]}([s,r])$ does not achieve 100\%.
\label{thm:hiddeninfo}
\end{theorem}
\begin{proof}

We will construct a $\kq$ such that $f$ does not achieve 100\% accuracy. By Assumption \ref{assump:ansdiversity}, $\forall r$ we have that there is at least one $s$ such that $(s, r, a') \in D_\text{pre}$ where $a' \neq \argmax_{a \in \mathcal{A}} \transpose{\embed{a}} \valuematrix \embed{r}$.

We will refer to $D_{min} = \{(s,r,a) \in D_\text{pre} \, |\, a  \neq \argmax_{a \in \mathcal{A}} \transpose{\embed{a}} \valuematrix \embed{r}\}$ (i.e. the set of $(s,r,a)$ triples whose answers are not the most strongly encoded with respect to relation token $r$). We will show that we can construct \kq (without modifying $W^{V}$) such that all points in $D_{\text{min}}$ are incorrectly responded to.

Consider $(s,r,a) \in D_{\min}$. At balanced attention, we have that $\argmax_{a' \in \mathcal{A}} \transpose{\embed{a'}} (\valuematrix \embed{s} + \valuematrix \embed{r}) = a$ because the one-layer transformer achieves 100\% accuracy. However we also have that $\argmax_{a' \in \mathcal{A}} \transpose{\embed{a'}} \valuematrix \embed{r} \neq a$ by the construction of $D_{\min}$. Denote the last-token attention scores given an attention matrix $W^{QK}$ as $\begin{bmatrix} \textrm{Att}_{s} \\ \textrm{Att}_{r} \end{bmatrix} = \softmax{\begin{bmatrix} \transpose{\embed{s}} \kq \embed{r}\\ \transpose{\embed{r}} \kq \embed{r}\end{bmatrix}}$. We have that if $\text{Att}_{s} (\transpose{\embed{a}} \valuematrix \embed{s}) \leq \text{Att}_{r}(\max_{a' \in \mathcal{A}} \transpose{\embed{a'}} \valuematrix \embed{r} -  \transpose{\embed{a}} \valuematrix \embed{r})$ then $\ntp{\valuematrix}{\kq}{s}{r} \neq a$.  We can see this by rearranging this inequality, yielding
\begin{equation*}
\transpose{\embed{a}}( \text{Att}_{s} \valuematrix \embed{s} + \text{Att}_{r} \valuematrix \embed{r}) \leq \text{Att}_{r} \transpose{\embed{\tilde{a}}}\valuematrix \embed{r}) \leq\transpose{\embed{\tilde{a}}}(\text{Att}_{s} \valuematrix \embed{s} +  \text{Att}_{r} \valuematrix \embed{r}).
\end{equation*}
where $\tilde{a} = \argmax_{a \in \mathcal{A}} \transpose{\embed{a}} \valuematrix \embed{r}$. This implies that $\ntp{\valuematrix}{W_{KQ}}{s}{r} \neq a$.
We will term $\max_{a' \in \mathcal{A}^{r}} \transpose{\embed{a'}} \valuematrix \embed{r} -  \transpose{\embed{a}} \valuematrix \embed{r}$ as a relation specific constant $d$. Then we can achieve erasure of the fact $(s,r)$ by ensuring 
\begin{equation*}
\softmax{\begin{bmatrix}\transpose{\embed{s}} \kq \embed{r}\\ \transpose{\embed{r}} \kq \embed{r} \end{bmatrix}}_{0} \leq \softmax{\begin{bmatrix}\transpose{\embed{s}} \kq \embed{r}\\ \transpose{\embed{r}} \kq \embed{r} \end{bmatrix}}_{1}\frac{d}{\transpose{\embed{a}} \valuematrix \embed{s}},
\end{equation*}

 Since both the terms $\transpose{\embed{s}} \kq \embed{r}$ and $\transpose{\embed{r}} \kq \embed{r}$ are free variables, we will fix $\transpose{\embed{r}}\kq \embed{r} = 0$ without loss of generality and compute the required constraint on the term $\transpose{{\embed{s}}} \kq \embed{r}$. For convenience, we will use the abbreviation $c = \transpose{\embed{s}} \kq \embed{r}$.

Substituting these simplifications, we have the following inequality
\begin{equation*}
\frac{\exp{c}}{\exp{c}+1} \leq \frac{1}{\exp{c}+1} \frac{d}{\transpose{\embed{a}} \valuematrix \embed{s}}.
\end{equation*}

We have that setting $c = \transpose{\embed{s}} \kq\embed{r} \leq \log \frac{d_{r}}{\transpose{\embed{a}}\valuematrix \embed{s}}$ achieves this. This confirms our intuition that as fact salience $\transpose{\embed{a}}\valuematrix \embed{s}$ becomes large, we must set the entry $\transpose{\embed{s}} \kq \embed{r}$ increasingly negative to ensure that an incorrect answer is output.
\end{proof}

\subsection{Results on Token Learning}
\label{app:tokenlearning}
In this section, we establish some theory relating to the representations of tokens after pretraining. First, we prove the following theorems regarding bounded softmax functions. 

\begin{theorem}[Softmax on $\ell_{\infty}$ bounded vectors]
Consider $x \in \mathbb{R}^{d}$ and suppose $\norm{x}_{\infty} \leq C$. Then $\max_{i} (\sigma(x))_{i} \leq \frac{e^{2k}}{d-1}$ and $\min_{i} (\sigma(x))_{i} \geq \frac{e^{-2k}}{d}$
\label{thm:softmax}
\end{theorem}
\begin{proof}
$\sigma(x)_{i} = \frac{\exp(x_{i})}{\sum\limits_{j \in d} \exp(x_{j})} \leq \frac{\exp(C)}{\exp(C)+(d-1)\exp(-C)} = \frac{\exp(2C)}{\exp(2C)+(d-1)} \leq \frac{\exp(2C)}{d-1}$. Likewise $\sigma(x)_{i} \geq \frac{\exp(-C)}{\exp(-C)+(d-1)\exp(C)} = \frac{\exp(-2C)}{\exp(-2C)+(d-1)} \geq \frac{\exp(-2C)}{d}$.
\end{proof}
Now we will prove a result on the activation of a token $t$,  $\valuematrix \embed{t}$ when trained by gradient descent with Cross Entropy loss with learning rate $\epsilon$ and updated $N$ times. We will make the following assumptions: 
\begin{assumption}[Attention Matrix Bounded]
$\forall t \in \mathcal{T}$ 
$\norm{\kq \embed{t}}_{\infty} \leq \frac{C_{KQ}}{2}$
\label{assump:bndatt}
\end{assumption}

\begin{assumption}[Value Matrix Bounded]
$\forall t \in \mathcal{T}$ 
 $\norm{\valuematrix \embed{t}}_{\infty} \leq \frac{C_V}{2}$
\label{assump:bndv}
\end{assumption}

We require that Assumptions \ref{assump:bndv} and \ref{assump:bndatt} hold throughout the training trajectory we consider. 
Now, consider a fixed token $t^{*} \in \mathcal{T}$ and a pretraining dataset $D_\text{pre}$. Let $D_{t} = \{(t^{*}, t_{1}, t_{a_1}),..., (t^{*}, t_{n}, t_{a_{n}})\}$ be all examples in $D_\text{pre}$ where $t^{*}$ occurs. Furthermore consider that $t_{a_{1}},...,t_{a_{n}} \in \mathcal{T}^{a}$ and that $|\mathcal{T}^{a}|=k$, i.e. that $T^{*} = \{t^{a}_{1},...,t^{a}_{k}\}$. Finally let $n_{i}$ denote the number of time $t^{a}_{i}$ appears in $D_{pre}$. Now we are ready to state the theorem. 
\begin{theorem}[Token Training Dynamics]
Consider training a one-layer transformer with parameters $ [\valuematrix,\kq]$ starting at $0$ initialization with batch size $1$ and learning rate $\epsilon$. Assume throughout training, we satisfy Assumptions \ref{assump:bndatt} and \ref{assump:bndv}. Then after one pass through $D_\text{pre}$ we have that $\forall t^{a}_{i} \in \mathcal{T}^{a}$ $ \transpose{\embed{t_{a_{n}}}} (W^{V} \embed{t^{*}}) \geq (\frac{n_{i} \exp(-C_{KQ})}{2} - n\frac{\exp(C_{V})}{|\mathcal{T}|-1} )\epsilon$. 
\label{thm:tok}
\end{theorem}
\begin{proof}
    We have that the single gradient step update for $W_{V}$ on example $(t^{*}, t_{i}, t^{a}_{j})$ can be written:
    \begin{equation*}
    W_{V}^{T+1} = W_{V}^{T} + (\text{Att}_{t^{*}}(\embed{t_{j}^{a}} - \ntp{\valuematrix}{\kq}{t^{*}}{ t_{i}})\transpose{\embed{t^*}} + \text{Att}_{t_{i}} (\embed{t_{j}^{a}} - \ntp{\valuematrix}{\kq}{t^{*}}{t_{i}}) \transpose{\embed{t_{i}}})
    \end{equation*}
    where we denote the attention scores respectively on $t^{*}$ and $t_{i}$ as $\text{Att}_{t^{*}}$ and $\text{Att}_{t_{i}}$. Since we are primarily focused on the value matrix projection of $t^{*}$, we will discard the $t_{i}$ term in what follows (due to orthogonality).

    We will first establish an upper bound on the pre-softmax output of the transformer, $Z$, in terms of the $\ell_{\infty}$ norm. By expanding and applying the triangle inequality we have  
    \begin{equation*}\norm{Z}_{\infty} = \norm{\att{t^*} \valuematrix\embed{t^{*}} + \att{t_{i}} \valuematrix\embed{t_{i}}}_{\infty} \leq \att{t^{*}} \norm{ \valuematrix \embed{t^{*}}}_{\infty} + \att{t_{i}} \norm{\valuematrix \embed{t_{i}}}_{\infty}. 
    \end{equation*}
    
    By Assumption \ref{assump:bndv} we have that $\norm{\valuematrix \embed{t^{*}}}_{\infty} \leq \frac{C_{V}}{2}$ and $\norm{W^{V} \embed{t_{i}}}_{\infty} \leq \frac{C_{V}}{2}$. This implies that 
    \begin{equation*}
    \norm{Z}_{\infty} \leq \att{t^{*}} \norm{\valuematrix \embed{t^*}}_\infty + \att{t_{i}} \norm{\valuematrix \embed{t_{i}}}_\infty  \leq  \frac{C_{V}}{2} (\att{t^*}+\att{t_i}) = \frac{C_{V}}{2}
    \end{equation*}
    where the final equality comes from softmax property that $\att{t^*} + \att{t_{i}} = 1$
    
    Next observe that $\ntp{\valuematrix}{\kq}{t^*}{t_{i}} = \softmax{Z}$. We can then apply Theorem \ref{thm:softmax} to upper and lower bound the components of $\ntp{\valuematrix}{\kq}{t^*}{t_{i}}$. In particular, we have that $\forall i \in [0,|\mathcal{T}|]$
    \begin{equation*}
        \frac{\exp(-C_V)}{|\mathcal{T}|} \leq (\ntp{\valuematrix}{\kq}{t^{*}}{t_{i}})_{i} \leq \frac{\exp(C_V)}{|\mathcal{T}|-1}.
    \end{equation*}

    Now, we will examine the attention term, we have that 
    \begin{equation*}
    \begin{bmatrix} \att{t^{*}} \\ \att{t_{i}}  \end{bmatrix} = \softmax{\begin{bmatrix} \transpose{\embed{t^{*}}} \kq \embed{t_{i}} \\ 
    \transpose{\embed{t_{i}}} \kq \embed{t_{i}}
    \end{bmatrix}}.
    \end{equation*}

    By Assumption \ref{assump:bndatt} we have that 
    \begin{equation*}
    \norm{\begin{bmatrix} \transpose{\embed{t^{*}}} \kq \embed{t_{i}} \\ 
    \transpose{\embed{t_{i}}} \kq \embed{t_{i}} 
    \end{bmatrix}}_{\infty} \leq \frac{C_{KQ}}{2}
    \end{equation*}
    and thus, by Theorem \ref{thm:softmax} we have that for $t \in \{t^{*}, t_{i}\}$ we have that 
    \begin{equation*}
    \frac{\exp(-C_{KQ})}{2} \leq \att{t} \leq \exp(C_{KQ}).
    \end{equation*}
    
    Next for any $t_{i}^{a} \in \mathcal{T}^{a}$, the co-ordinate $\transpose{\embed{t_{i}^{a}}}(\valuematrix \embed{t^{*}})$ receives $n_{i}$ updates of the form $+\epsilon \att{t^{*}}$ and $n$ updates of the form $-\ntp{\valuematrix}{\kq}{t^{*}}{ t'} \att{s}$. We can then lower bound the quantity  $\transpose{\embed{t_{i}^{a}}}(\valuematrix \embed{t^{*}}) \geq n_{i}\epsilon \frac{\exp(-C_{QK})}{2} - n\frac{\exp(C_{V})}{|\mathcal{T}|-1} \epsilon$ (where we have upper bounded the attention score in the second term by $1$), thereby yielding the desired claim. 
\end{proof}
Now, we are ready to prove the Theorem \ref{thm:tokenlearning}, as a straightforward application of Theorem \ref{thm:tok}. We restate Theorem \ref{thm:tokenlearning} below for convenience. 
\begin{theorem}[Lower bound on fact salience]
Consider pretraining $\ntp{\valuematrix}{\kq}{s}{r}$ on a dataset $D_\text{pre}$ of size $N$ for one epoch with learning rate $\epsilon$. Suppose that the $\norm{W_{KQ}}_{\infty}<C_{QK}$ and $\norm{W_{V}}_{\infty}<C_{V}$ throughout training. Suppose that the combination $(s,r)$ appears $n$ times and $s$ appears no more than $s$ appears no more than $n_{tot}<n\frac{(|\mathcal{T}|-1)\exp(-C_{KQ})}{2\exp(C_{V})}$ times. Then  ($\transpose{\embed{a}})(W_{V} \embed{s}) \geq nc_{1}\epsilon$ where $c_{1} > 0$.
\label{thm:tokenlearningapp}

\end{theorem}
\begin{proof}
We can simply treat the set $\{a|(s,r,a) \in D_{pre}\}$ as $\mathcal{T}^{a}$ in the theorem above. Then we have that $n_{i} = n$ (for the purposes of Theorem \ref{thm:tok}) and finally that $n<n_{tot}$. This implies that $\transpose{\embed{a}}(W^{V} \embed{s}) \geq (n\frac{\exp(-C_{KQ})}{2}-n_{tot}\frac{\exp(C_{V})}{|\mathcal{T}-1|}) \epsilon$. We have that by the hypothesis of the theorem, $(n\frac{\exp(-C_{KQ})}{2}-n_{tot}\frac{\exp(C_{V})}{|\mathcal{T}-1|}) >0$ which gives the desired result

\end{proof}

 \subsection{Results on Attention Dynamics}
 \label{app:att}
 Now, we examine the process by which attention is learned in the one-layer transformer. Preliminarily, we have that the update corresponding to the attention matrix $W_{KQ}$ on an abstract gradient update $(s,r,a)$ is
 \begin{equation*}
 \begin{split}
 -\frac{\partial L}{W_{KQ}} = c \underbrace{\transpose{((\embed{a}-\ntp{\valuematrix}{\kq}{s}{r})} (\valuematrix \embed{r})}_{\text{correlation of r with update}} \underbrace{(\embed{r}\transpose{\embed{r}} - \embed{s}\transpose{\embed{r}})}_{\text{increase attention on r}} + \\\underbrace{\transpose{(\embed{a}-\ntp{\valuematrix}{\kq}{s}{r})} (\valuematrix \embed{s})}_{\text{correlation of s with update}}\underbrace{(\embed{s}\transpose{\embed{r}}-\embed{r}\transpose{\embed{r}})}_{\text{increase attention on s}}
     \end{split}
\end{equation*}
where $c>0$, arising from the gradient of the softmax term. For subsequent convenience, we will condense the update as 
\begin{equation*}
-\frac{\partial L}{W_{KQ}} = \transpose{(\embed{a}-\ntp{\valuematrix}{\kq}{s}{r})}(\valuematrix \embed{r} -\valuematrix \embed{s}) (\embed{r}\transpose{\embed{r}}-\embed{s}\transpose{\embed{r}}).
\end{equation*}
On this basis, we now prove Theorem \ref{thm:attndynamics}, first restating it for convenience.

\begin{theorem}[Factuality vs. Nonfactuality Inducing Gradients]
When finetuning on a fact $(s,p_{r})$, if $s_{rel} - p_{rel}<0$ then the attention update  $-\frac{\partial L}{\partial W_{KQ}}$  globally decreases the attention on all $s' $ when prompting with $(s', p_r)$, whereas when $s_{rel} - p_{rel}>0$, $-\frac{\partial L}{\partial W_{KQ}}$ globally increases the attention on all $s'$.
\label{thm:attndynamics}
\end{theorem}

\begin{proof}
Beginning with the update of the query-key matrix, we observe that when $s_{rel} - p_{rel}<0$, then we have (denoting the \emph{post-update} value matrix as $\kq^{'}$). 
\begin{equation*}
\transpose{\embed{r}} \kq^{'} \embed{r} = \transpose{\embed{r}} \kq \embed{r}+ t
\end{equation*}
where $t>0$. This follows from the update rule: 
\begin{equation*}
\kq^{'} = \kq + t(\embed{r}\transpose{\embed{r}} - \embed{s}\transpose{\embed{r}})
\end{equation*}
since the embeddings $\phi(r)$ are unit norm and we can omit the impact of the  $\embed{s}\transpose{\embed{r}}$ term by orthogonality. 

Next, to show the global nature of the update, we consider an arbitrary subject token $s'$ and compare the attention before and after the update on the prompt $[s',r]$. We will denote the pre-update subject token attention as $\att{s'}^{t}$ and the post-update subject token attention as $\att{s'}^{t+1}$. Then, we have
\begin{equation*}
\att{s'}^{t-1} = \frac{\exp(\transpose{\embed{s'}} \kq \embed{r})}{\exp(\transpose{\embed{s'}} \kq \embed{r}) + \exp(\transpose{\embed{p_{r}}} \kq \embed{p_{r}})}
\end{equation*}
and the attention on the subject after the update is 
\begin{equation*}
\att{s'}^{t} = \frac{\exp(\transpose{\embed{s'}} \kq \embed{r})}{\exp(\transpose{\embed{s'}} \kq \embed{r}) + \exp(\transpose{\embed{p_{r}}} \kq \embed{p_{r}} + t)}
\end{equation*}
Thus, by the monotonicity of $\exp$ and the fact that $f(x) = x^{-1}$ is decreasing, we have that $\att{s'}^{t-1} > \att{s'}^{t} $. In the second case $s_{rel}-p_{rel}>0$, the update to $\kq$ is $-t\embed{r}\transpose{\embed{r}}$ and we have the result by the same reasoning. 
\end{proof}

\section{Experimental Details}
\label{sec:expdetails}
\subsection{Hyperparameters and Tuning}
\begin{table}
\caption{Large Language Model Hyperparameters}
\label{table:hparam_llm}
\centering
\begin{tabular}{l l l}
    \toprule
    {Hyperparameter} &   {Range} \\
    \midrule                                           
    Learning Rate  &  1e-5, 1e-4, 1e-3     \\
    Weight Decay& 1e-6, 1e-5, 1e-4, 1e-3, 1e-2 \\
    (LoRA rank, LoRA $\alpha$)& (8,16), (16,32), (32, 64), (64,128)  \\
    LoRA & True, False\\
    \midrule
\end{tabular}
\end{table}

Across all experiments, we tune the following hyper-parameters on the ranges shown in Table \ref{table:hparam_llm} on a held-out validation set. The LoRA entry refers to selecting whether LoRA or full-finetuning is used. In all experiments, we found that LoRA achieved better validation performance than full-finetuning. We report the performance after tuning on a held-out validation set in all experiments. Tuning is performed individually for each fine-tuning dataset. 
\subsection{Evaluation}
\label{app:eval}
For short answer questions, we normalize the LLM output (generating up to 10 tokens), \texttt{llm\_out} by the following:

\texttt{llm\_out\_norm = llm\_out.lower().rstrip().lstrip()}

Given a list of possible ground-truth answers in the dataset \texttt{gt\_list}, we compute whether the LLM response is correct as:

\texttt{any([x in llm\_out\_norm for x in gt\_list])}

We consider the LLM's output on a multiple choice dataset as:

\texttt{torch.argmax(out\_scores[token\_set])}

where \texttt{out\_scores} are the next token prediction scores on and \texttt{token\_set} are the indices of the answer choice tokens (i.e. A, B, C, or D).
\label{sec:hparamtunings}
\subsection{Attention Mechanism Analysis}
\label{sec:attmechanismdetails}
We denote the subject attention as the maximum attention score over the tokens corresponding to the subject entity (relative to the final prompt token), as proposed in \cite{yuksekgonul2023attention}. The attention score for each layer is computed by averaging over all attention heads in that layer. On the left panel, we show the average attention score to the subject entity averaged over the PopQA test set. We observe, past the initial few layers, that the attention to the subject significantly drops in the model fine-tuned on FT-Bottom, providing mechanistic evidence of our hypothesis in a large language model. 

\section{Additional Experimental Results}
\label{sec:addexp}

\subsection{Numerical Experiments with One-Layer Transformer}
\label{app:numerical_experiment_1_layer}
\begin{figure}
    \centering
    \subfigure[Accuracy]{
\centering 
    \includegraphics[width =0.3\textwidth]{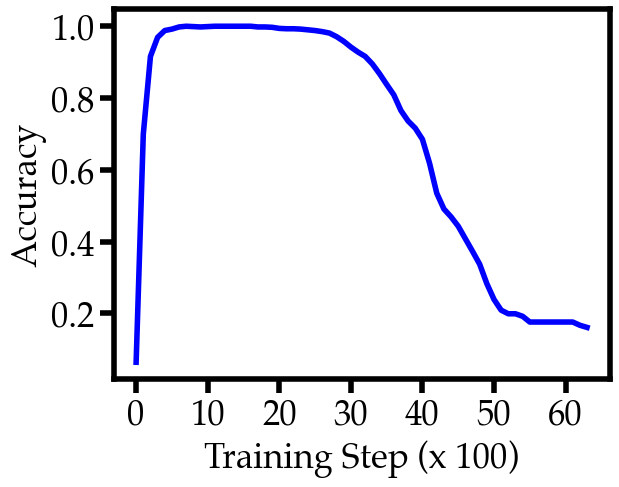}
}
\hfill
\subfigure[Subject Attention]{
    \centering
    \includegraphics[width = 0.3\textwidth]{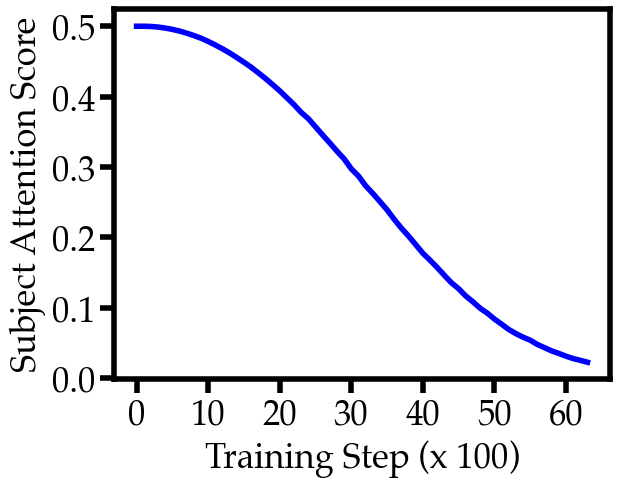}
}
\hfill
\subfigure[Relation Attention]{
\centering
    \includegraphics[width = 0.30\textwidth]{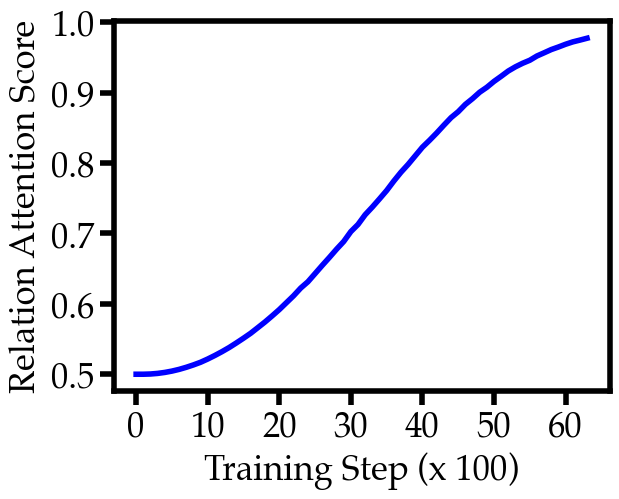}
}\hfill
\caption{\textbf{Numerical Results on One-Layer Transformer} We consider the same pretraining data distribution as previously but with training a one-layer transformer rather than a multi-layer transformer. This represents an exact simulation of the setting of our theoretical model from Section \ref{sec:thy}. In (a) we plot the factual accuracy as a function of the finetuning step. In (b), we plot the attention score on the subject token, and in (c) the attention score on the relation token.}
\end{figure}
In this section, we provide numerical evidence of our 1-layer transformer theory introduced in Section 4. We replicate our experiments in the synthetic language seen in Section 3 in a one-layer transformer. As introduced in \cite{li2023transformers}, we consider the initialization setting in which both the value and query-key matrices are initialized from a normal distribution with very low variance $(0.001)$. 

As a result of this configuration, the attention matrix initially is not updated (i.e. steps 1-10), because the magnitude of the update depends on the entries of the value matrix which are initially small. Later on in training, the attention matrices begin to update (in this case once all facts are learned). This results in the attention shifting. This can be seen as an example of \emph{two-stage} training dynamics for a one-layer transformer, as is observed in \cite{li2023transformers}.

In Figure \ref{app:numerical_experiment_1_layer}, we plot the attention to the relation token as a function of finetuning training steps. We observe that the attention to the relation token grows throughout training, eventually approaching $1$ (i.e. the one-layer transformer entirely ignores the subject token) and similarly, the subject attention declines to $0$. As a result of this, we see that the performance of the model \emph{steeply declines} as training continues, demonstrating "hidden information" caused by the imbalanced attention that we present in our theory.

\subsection{Additional Attention Maps}
Due to space constraints, we include additional attention maps of models trained on \texttt{FT-Top} and \texttt{FT-Bottom} in Figure \ref{fig:appadlattmap}. Overall, we observe a reduction of attention scores when evaluating on the \texttt{FT-Bottom} model, as we describe in Section \ref{sec:real}.
\begin{figure}
\centering
\includegraphics[width=0.7\textwidth]{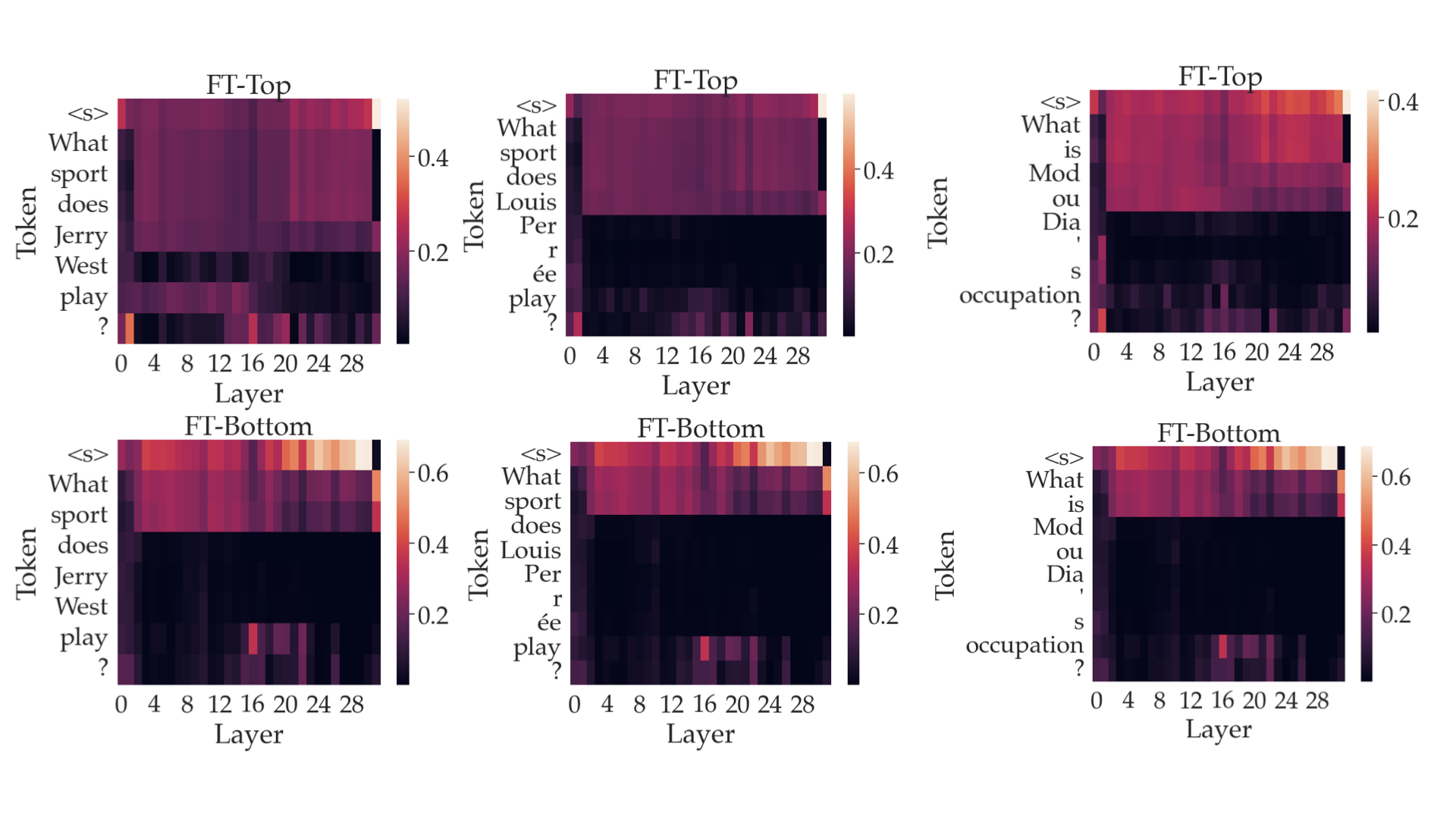}
\caption{\textbf{Additional Attention Maps} We include additional example-specific attention maps for models fine-tuned with \texttt{FT-Top} and \texttt{FT-Bottom}, respectively. Overall, we find the trends predicted by our analysis of the one-layer transformer continue to hold. In particular, 
 subject attention is markedly reduced for models that are fine-tuned on \texttt{FT-Bottom}.}
\label{fig:appadlattmap}
\end{figure}

\end{document}